\documentclass{iopart}

\usepackage{iopams}

\expandafter\let\csname equation*\endcsname\relax

\expandafter\let\csname endequation*\endcsname\relax

\usepackage{amsmath}

\usepackage{amssymb,amsthm}
\usepackage{amsopn,amstext}
\usepackage{graphicx}
\usepackage{color}
\usepackage{fullpage}
\usepackage{stmaryrd}
\usepackage{hyperref}
\usepackage{pgfplots}

\renewcommand{\d}{{\mathrm{d}}}

\newcommand{\spt}{{\mathrm{spt}}}
\newcommand{\sign}{{\mathrm{sgn}}}
\newcommand{\meas}{{\mathcal{M}}}
\newcommand{\measp}{{\mathcal{M}_+}}

\newcommand{\diss}{{{d}}}
\newcommand{\R}{{\mathbb{R}}}

\newcommand{\N}{{\mathbb{N}}}
\newcommand{\Z}{{\mathbb{Z}}}

\newcommand{\imageParamVec}{\mathbf p}
\newcommand{\diff}{{\mathcal{D}}}
\newcommand{\den}{u}
\newcommand{\img}{g}
\newcommand{\Den}{U}
\newcommand{\Img}{G}
\newcommand{\Motion}{\mathbf W}
\newcommand{\motion}{\mathbf w}
\newcommand{\Poisson}{\mathcal P}
\newcommand{\Gauss}{\mathcal N}
\newcommand{\data}{{\mathrm{data}}}
\newcommand{\elec}{q}
\newcommand{\Elec}{Q}
\newcommand{\elecP}{\rho}
\newcommand{\white}{r}
\newcommand{\White}{R}
\newcommand{\shift}{w}
\newcommand{\shiftVec}{\motion}
\newcommand{\bumpCoeff}{c}
\newcommand{\bumpWidth}{\omega}
\newcommand{\centerPos}{y}
\newcommand{\inputImage}{\mathbf\img}
\newcommand{\reconstImage}{\den[\imageParamVec]}
\newcommand{\numAtoms}{L}
\newcommand{\numFitParams}{J}
\newcommand{\atomIndex}{l}
\newcommand{\numPixelsX}{{N_1}}
\newcommand{\numPixelsY}{{N_2}}

\newcommand{\ie}{i.\,e.\ }
\newcommand{\eg}{e.\,g.\ }

\newtheorem{theorem}{Theorem}
\newtheorem{lemma}{Lemma}
\newtheorem{remark}{Remark}

\graphicspath{{./figures/}}

\begin{document}

\title{Joint denoising and distortion correction of atomic scale scanning transmission electron microscopy images}

\author{Benjamin~Berkels$^1$ and Benedikt Wirth$^2$}
\address{
$^1$ AICES Graduate School, RWTH Aachen University, Germany\\
$^2$ Applied Mathematics, University of M\"unster, Germany
}

\begin{abstract}
Nowadays, modern electron microscopes deliver images at atomic scale. The precise atomic structure encodes information about material properties. Thus, an important ingredient in the image analysis is to locate the centers of the atoms shown in micrographs as precisely as possible.
Here, we consider scanning transmission electron microscopy (STEM), which acquires data in a rastering pattern, pixel by pixel. Due to this rastering combined with the magnification to atomic scale, movements of the specimen even at the nanometer scale lead to random image distortions that make precise atom localization difficult. Given a series of STEM images, we derive a Bayesian method that jointly estimates the distortion in each image and reconstructs the underlying atomic grid of the material by fitting the atom bumps with suitable bump functions. The resulting highly non-convex minimization problems are solved numerically with a trust region approach. Well-posedness of the reconstruction method and the model behavior for faster and faster rastering are investigated using variational techniques.
The performance of the method is finally evaluated on both synthetic and real experimental data.
\end{abstract}

\section{Introduction}

Nowadays, imaging techniques like Transmission Electron Microscopy (TEM) allow to acquire images of materials at atomic resolution. The precise atomic configuration that can be identified from those images allows to infer information about material properties. In this article, we concentrate on a variant of TEM, described next.

\subsection{STEM}

An important variant of TEM is the so-called Scanning Transmission Electron Microscopy (STEM). STEM acquires images by moving a focused electron probe over a sample in line-by-line, pixel-by-pixel manner. In this paper, we consider high-angle annular dark-field (HAADF) STEM, which essentially counts at each pixel the electrons from the beam that are deflected by the sample and arrive at a circular annulus centered around the sampling position. A reason for the popularity of HAADF-STEM~\cite{BiBlBl12} is that the number of electrons counted at a pixel is proportional to the atomic number of the material at the corresponding position, which allows for a direct interpretation of the measured intensities. This property comes at a price. The sequential pixel-by-pixel rastering process combined with the magnification to atomic scale makes movements of the sample at the nanometer scale (\eg induced by environmental and instrumental disturbances) cause distortions in STEM images (see Figure~\ref{fig:STEMAndSimulImage} for an example).
Most visibly, there are characteristic discontinuous horizontal distortions of the depicted atoms~\cite{JoNe13}. Furthermore, low frequency sample drift induces smooth, but possibly non-rigid deformations of the depicted atomic lattice.
The longer the exposure (time spent at a pixel), the bigger the distortions. Thus, a natural way to reduce the distortions is to decrease the exposure time. Unfortunately, the signal-to-noise ratio (SNR) is inversely proportional to the exposure time. To get both small distortions and a good SNR, an often employed strategy is to acquire several images with short exposure instead of a single image with long exposure. Of course, then one has to reconstruct the underlying atom configuration from a series of images.
Due to its special nature, STEM imaging calls for tailored image processing techniques.
Finally, in many STEM applications, the reconstructed high quality image is no end in itself, but just an intermediate step towards a characterization of the depicted material. This characterization is a description of the atomic grid of the material, \ie the position of the atoms (which are actually atomic columns, since STEM shows a 2D projection of the 3D material structure) and their atomic number.

\subsection{The proposed reconstruction model}
We now briefly describe our ansatz for the inverse problem of identifying atom positions and other parameters from a STEM measurement.
STEM aims to sample the material by obtaining a measurement $\img_i$ (essentially the number of electrons deflected from the electron beam) at a discrete set of locations $x_i\in\R^2$, $i=1,\ldots,N$.
To this end, the beam is positioned at $x_i$ at time $t_i$, however, it will be slightly displaced relative to the sample by a stochastic motion $w_i\in\R^2$, for instance due to temperature fluctuations.
Given the sequence of measurements $\img_i$, we try to reconstruct the underlying material distribution of the sample
(essentially the atom numbers and positions, encoded by a vector $\mathbf p$ from a set $A$ of reasonable parameters)
as well as the stochastic perturbations $w_i$ of the measurement locations by minimizing the functional
\begin{equation}\label{eqn:model}
E[\motion,\mathbf p]=
\sum_{i=1}^N\diss(\img_i,\Delta t\den[\mathbf p](x_i+w_i))+\frac1{2\diff}\sum_{i=1}^N\frac{|w_i-w_{i-1}|^2}{(t_i-t_{i-1})}
\end{equation}
over $\motion=(w_1,\ldots,w_N)\in(\R^2)^N$ and $\mathbf p\in A$.
This variational approach will be rigorously derived from a Bayesian model in the first part of this article, however, the rough intuition behind it is easily explained.
The expected number of deflected electrons at any position $x\in\R^2$ depends on the atom configuration $\mathbf p$ in the sample and can easily be calculated as a function $\Delta t\den[\mathbf p]:\R^2\to[0,\infty)$
(the factor $\Delta t>0$ indicates the resting time over position $x_i$, which the total number of deflected electrons should be proportional to).
A deviation of a measurement $\img_i$ at time $t_i$ from the expected number $\Delta t\den[\mathbf p](x_i+w_i)$ at the (perturbed, true) location $x_i+w_i$ is penalized by a dissimilarity $\diss$,
for which we will essentially consider the Kullback-Leibler divergence.
In addition, the second sum penalizes a change between consecutive perturbations $w_{i-1},w_i$ with the idea that the stochastic perturbation cannot change too rapidly.
The constant $\diff$ just weights the importance of that regularization.

Apart from applying this model to artificial and real STEM data,
we also analyze how our reconstruction behaves as the electron beam samples the material at a faster and faster rate, thereby allowing more measurement locations during a fixed measurement time.
This requires an understanding of how the stochastic signal of deflected electrons behaves as more and more measurement locations are scanned.
Such an understanding can be obtained using the same tools as in stochastic homogenization.
The reconstruction model in the limit of infinite sampling speed (which corresponds to a time-continuous or even infinitely fast electron beam motion) can then be derived via a $\Gamma$-convergence analysis.
The result depends on the scan path; using space-filling curves one can obtain a reconstruction model for scans sampling the whole material,
while for instance row-wise scans with more and more rows lead to a tomographic-type model, in which only the average signal per row is used.

Due to the nonconvexity and relatively high complexity of the variational model, a very good initialization is vital for accurate reconstructions.
For this purpose, the model is reduced to a less accurate but much simpler convex optimization problem, which allows a fast and globally optimal solution.
Essentially, the reduced model represents a deconvolution problem similar to the ones considered and analyzed for instance recently in \cite{BrPi13,DuPe15}.

In numerical experiments, our proposed method robustly identifies all atom positions with high accuracy.
Furthermore, the conducted limit analysis gives indications as to what are reasonable scan paths and how measurement parameters should be chosen in relation to the stochastic noise present during the measurement.

\subsection{Related work}
The development of special processing methods tailored to the characteristic STEM properties is an active research topic mostly in the field of electron microscopy, but to some degree also in mathematics.
In \cite{JoNe13}, Jones and Nellist study sources of STEM distortion and propose an algorithm to correct these distortions on individual images. Kimoto et al.\ \cite{KiAsYu10} propose to use averaging and rigid registration on a series of STEM images to obtain a high quality image of the underlying material, which achieves considerably better precision than previous methods working on just a single image.
Note that we use the term \emph{precision} here in the way it is used in the electron microscopy community. There, it is a measure how precisely atom centers can be located (for a detailed definition see Section~\ref{sec:results}). The quality of the average image of a series can be further improved significantly with non-rigid registration~\cite{BeBiBl13}. As of now, to the best of our knowledge, this method still achieves the best reported precision on STEM images in the literature~\cite{YaBeDa14}. Since it uses a smooth deformation model, this method cannot fully correct the horizontal STEM distortions. Recently, Jones et al.\ \cite{JoYaPe15} proposed a non-rigid, non-smooth registration algorithm that also aims at correcting the horizontal distortions.
In \cite{BaBoBr16}, De Backer et al.\ propose a model-based estimation of STEM images with Gaussian bumps that is specifically designed to analyze a large field of view. A key concept of this approach is the segmentation of the image into smaller sections, which avoids the simultaneous estimation of the non-linear fit parameters to reduce the computational complexitity.

The outline of the article is as follows.
In Section~\ref{sec:model}, we derive the variational model \eref{eqn:model} and in particular analyze its well-posedness and further properties in Section~\ref{sec:variationalModel}.
Section~\ref{sec:limitModels} then examines the limit models for a continuum of sampling locations with main results Theorem~\ref{thm:GammaLimit} as well as Theorem~\ref{thm:noGammaLimit} including the subsequent Remark~\ref{rem:tomography}.
Section~\ref{sec:ConvAtomIdent} considers a reduced reconstruction model necessary to initialize the numerical optimization,
while Sections~\ref{sec:numerics} and \ref{sec:results} describe the numerical implementation and results.

Finally, for the reader's convenience, below we provide a reference list of the most important symbols used throughout.\\
\setlength\unitlength{.6\linewidth}%
\def\arraystretch{1.5}%
\begin{tabular}{ll}
$(x_1,\ldots,x_N),(t_1,\ldots,t_N)$& STEM measurement locations and times\\
$\Delta t,\Delta T$& dwell time (at a measurement location) and waiting time\\
$\Omega=[0,a]^2$ & scanned material region\\
$\mathbf\Img=(\Img_1,\ldots,\Img_N),\mathbf\img=(\img_1,\ldots,\img_N)$& \parbox[t]{\unitlength}{random variable and realization of STEM measurements at locations $(x_1,\ldots,x_N)$}\\
$\Den,\den$& random variable and realization of material density\\
$\Motion=(W_{\diff t_1},\ldots,W_{\diff t_N}),\motion=(w_1,\ldots,w_N)$& \parbox[t]{\unitlength}{random variable and realization of the accumulated sample motion up to times $t_1,\ldots,t_N$}\\
$\mathbf P,\mathbf p=(\centerPos_1,\ldots,\centerPos_\numAtoms,\bumpCoeff_1\ldots,\bumpCoeff_\numAtoms,\bumpWidth,o)$& \parbox[t]{\unitlength}{random variable and realization of the vector of fitting parameters for $\den$, consisting of atom positions $\centerPos_\atomIndex$, atom heights and widths $\bumpCoeff_\atomIndex$ and $\bumpWidth$, and background $o$}\\
$\numAtoms$, $\numFitParams$& number of atoms and of fitting parameters\\
$A$ & domain of fitting parameters\\
$\numPixelsX,\numPixelsY$ & number of pixels in horizontal and vertical direction\\
$f_X$& probability density function of random variable $X$\\
$W_t$& Brownian motion\\
$\diff$ & diffusion coefficient of stochastic motion\\
$\Poisson,\Gauss$ & Poisson and normal distribution\\
$\alpha,\mu,\sigma^2$ & gain factor, noise mean and noise variance of electron detector\\
$E$& proposed energy functional\\
$K$& number of input images or measurements\\
$\mathcal W,\mathcal G,\mathcal E$& time-continuum versions of $\motion,\mathbf\img,E$
\end{tabular}

\section{A Bayesian model for removal of spatial and image noise}\label{sec:model}
In this section, we describe the basic concept of STEM applied to an atomic crystal,
and we model the conditional probability of acquiring a particular image given the underlying atom distribution and (random) motion of the material sample.
This will then lead to a variational Bayesian model for recovering atom distribution and motion from STEM measurements.
Note that even though identifying the random sample motion is not of primary interest,
the model actually simplifies by explicitly including that motion as an unknown to be determined.

\subsection{Noise-free STEM model}
Let $\Omega=[0,a]^2$ denote the scanned material region,
and let $\den:\Omega\to[0,\infty)$ denote the idealized STEM signal
that one would obtain without any noise in the image acquisition (in particular no motion of the sample), if each point was scanned for a time interval of length $1$
($\den$ represents the intensity of electron deflection at each point, but for simplicity we shall think of it as a material density).

HAADF-STEM acquires a measurement at a sequence $x_1,\ldots,x_N\in\Omega$ of points in the sample
by moving to each single point $x_i$ and resting there for a fixed time interval $\Delta t$, the so-called \emph{dwell time},
during which the number of electrons is counted that are deflected from the electron beam at a certain angle.
The points $x_i$ are typically arranged along horizontal lines covering $\Omega$ and are traversed row-wise, but other scanning paths are possible as well.
Typically, the STEM dwell time is of the order of 10 microseconds.
The instrument may also spend some additional time in between the signal acquisition at two consecutive measurement positions:
Even though the time for moving the electron probe from one measurement position to the next lies at least two orders of magnitude below the dwell time (about 10 to 100 nanoseconds),
large movements (such as from the end of one row to the beginning of the next) cause the electron probe to wobble a little
so that the instrument waits a time $\Delta T$ of around 60 microseconds before starting the scan of the next position.

By $t_i\in\R$, $i=1,\ldots,N$, we shall denote the time points at which the signal acquisition at location $x_i$ is finished
(note that, setting $t_0=0$, we necessarily have $t_i-t_{i-1}\geq\Delta t$ for all $i=1,\ldots,N$).
Then, the motion of the electron beam over the sample can be described by the piecewise constant function
\begin{equation*}
x(t)=x_i\quad\text{for }t\in[t_{i-1},t_{i})\,.
\end{equation*}
The signal recorded by the instrument at position $x_i$ is denoted $\img_i$.
Without noise, we thus would have $\img_i=\Delta t\den(x_i)$ for $i=1,\ldots,N$.

However, there are multiple noise sources.
In particular, the sample is not stationary, but undergoes a random drift that can be modeled by Brownian motion.
Furthermore, the number of deflected electrons is not deterministic, but obeys a stochastic law.
As a result, the measured signal is a random variable, depending on the stochastic motion and electron deflection of the sample.

In the following, we denote random variables by capital letters and their realizations by the corresponding lower-case letters.
Vectors will be denoted by boldfont letters.
The probability density function of a random variable $X$ will be denoted $f_X$.

\subsection{A Bayesian probability model for observed material density and motion}\label{sec:Bayes}
Let us abbreviate $\mathbf\Img=(\Img_1,\ldots,\Img_N)$ to be the random variable of the STEM measurements at locations $(x_1,\ldots,x_N)$.
Given actual measurements $\mathbf\img=(\img_1,\ldots,\img_N)$,
we aim to recover the underlying true material density $\den$ and the accumulated sample motion $w_{i}$ at time points $t_i$, $i=1,\ldots,N$.
Here, $\den$ and $\motion=(w_1,\ldots,w_N)$ are just realizations of random variables $\Den$ and $\Motion$ to be described later.
Note that $\motion$ can be viewed as an auxiliary variable that will simplify the modeling, but it may also be viewed as a quantity of its own interest.
For instance, $\motion$ may contain information about systematic motion artifacts that occurred during the image acquisition and allow the microscopist to identify and eliminate the respective error sources.

To set up a corresponding variational model we would like to describe the conditional probability
of the measurements $\mathbf\img$ being produced with motion $\Motion=\motion$ and density $\Den=\den$.
By Bayes' theorem (see \eg \cite[Sec.\,1.1.6]{Le12}), the corresponding probability density function can be expressed as
\begin{equation*}
f_{\Motion,\Den}(\motion,\den\,|\,\mathbf\Img=\mathbf\img)
=\frac{f_{\mathbf\Img}(\mathbf\img\,|\,\Motion=\motion,\Den=\den)f_{\Motion,\Den}(\motion,\den)}{f_{\mathbf\Img}(\mathbf\img)}\,.
\end{equation*}
As our estimate of $\den$ and $\motion$ for a given image $\mathbf\img$, we will later use the so-called maximum a posteriori estimate, which is the pair $(\motion,\den)$ maximizing the above conditional probability.

Since the Brownian motion and the material density are independent, we have
\begin{equation*}
f_{\Motion,\Den}(\motion,\den)
=f_{\Motion}(\motion)f_{\Den}(\den)\,.
\end{equation*}
The following paragraphs derive expressions for $f_{\Motion}(\motion)$, $f_{\Den}(\den)$, and $f_{\mathbf\Img}(\mathbf\img\,|\,\Motion=\motion,\Den=\den)$.

\subsection{Brownian motion of the sample}
The sample motion is mainly due to random temperature fluctuations
 and thus can be modeled using two-dimensional Brownian motion, denoted by $W_t\in\R^2$ for $t\in\R$.
The actual random variable describing the sample motion is $W_{\diff t}$ with $\diff$ encoding the diffusion time scale of the motion.
An actual path, that is, a realization of $W_{\diff t}$ is denoted $w_t$.
Hence, the random variable of accumulated motion up to times $t_1,\ldots,t_N$ is
\begin{equation*}
\Motion=(W_{\diff t_1},\ldots,W_{\diff t_N})
\text{ with realizations }
\motion=(w_{1},\ldots,w_{N})\,.
\end{equation*}
By definition of Brownian motion, $W_{\diff t_i}-W_{\diff t_{i-1}}$ is distributed according to the normal distribution of mean $0$ and covariance $\diff(t_i-t_{i-1})I_2$ with the $2\times2$ identity matrix $I_2$,
for which we use the notation $W_{\diff t_i}-W_{\diff t_{i-1}}\sim\Gauss(0,\diff(t_i-t_{i-1})I_2)$.
Denoting the accumulated random motion up to the begin of the measurement by $w_0$ (by a coordinate shift we may also simply define $w_0$ to be zero), we thus have the probability density function
\begin{equation*}
f_{\Motion}(\motion)
=\prod_{i=1}^Nf_{(W_{\diff t_i}-W_{\diff t_{i-1}})}(w_i-w_{i-1})
=\frac1{\prod_{i=1}^N{2\pi\diff(t_i-t_{i-1})}}\exp\left(-\sum_{i=1}^N\frac{|w_i-w_{i-1}|^2}{2\diff(t_i-t_{i-1})}\right)\,.
\end{equation*}
\subsection{Model for the sample density}
Computer simulations based on forward models of STEM show \cite{Ki10} that the material density for a given sample has the form
\begin{equation*}
\den(x)=\den[\mathbf p](x)=\sum_{\atomIndex=1}^\numAtoms b[\bumpCoeff_\atomIndex](x-\centerPos_\atomIndex)+o\,,
\end{equation*}
where $o$ is a constant (or slowly varying) background gray level, $\numAtoms$ is the number of atom locations visible in $\Omega$,
$\centerPos_\atomIndex\in\Omega$ is the $\atomIndex$\textsuperscript{th} atom location,
$\mathbf p=(\centerPos_1,\ldots,\centerPos_\numAtoms,\bumpCoeff_1\ldots,\bumpCoeff_\numAtoms,o)$ is a vector of parameters,
and $b[\bumpCoeff_\atomIndex]:\R^2\to\R$ is the signal response of a single atom (or of multiple atoms stacked above one another in the atom lattice), parameterized by some coefficient $\bumpCoeff_\atomIndex\in\R^m$.
Since the number of atoms can be readily identified from a STEM measurement, we assume $\numAtoms$ to be fixed.
A good model for the response of a single atom seems to be a Gaussian bell function
of height $\bumpCoeff_{\atomIndex,1}$ (depending on the atom type and the number of stacked atoms) and width $\bumpCoeff_{\atomIndex,2}$ (which only depends on the employed magnification of the microscope),
that is, $m=2$, $\bumpCoeff_\atomIndex=(\bumpCoeff_{\atomIndex,1},\bumpCoeff_{\atomIndex,2})$ and
\begin{equation*}
b[\bumpCoeff_\atomIndex](x)=\bumpCoeff_{\atomIndex,1}\exp\left(\frac{-|x|^2}{2\bumpCoeff_{\atomIndex,2}^2}\right)\,.
\end{equation*}
We will use this particular $b$ throughout the article, but other functions are possible as well.
As noted above, the width $\bumpCoeff_{\atomIndex,2}$ only depends on the employed magnification and the instrument. In particular, it does not depend on the type of the corresponding atom. Thus, it is sufficient to treat the width as a single scalar unknown instead of having a separate width for each atom. Hence, the vector of parameters changes to
$\mathbf p=(\centerPos_1,\ldots,\centerPos_\numAtoms,\bumpCoeff_1\ldots,\bumpCoeff_\numAtoms,\bumpWidth,o)$ with $c_l\in\mathbb{R}$ and $\bumpWidth\in\mathbb{R}$.

The above implies that instead of working with the random variable $\Den$, describing the distribution of material densities $\den$,
we may just as well work with the random variable $\mathbf P$, describing the distribution of parameters $\mathbf p$, and we have
\begin{equation*}
f_{\Den}(\den)=\begin{cases}f_{\mathbf P}(\mathbf p)&\text{if }\den=\den[\mathbf p]\\0&\text{else}\end{cases}
\end{equation*}
and
\begin{equation*}
f_{\mathbf\Img}(\mathbf\img\,|\,\Motion=\motion,\Den=\den)
=\begin{cases}f_{\mathbf\Img}(\mathbf\img\,|\,\Motion=\motion,\mathbf P=\mathbf p)&\text{if }\den=\den[\mathbf p]\\0&\text{else.}\end{cases}
\end{equation*}
For the entries of $\mathbf P$, we assume (in lack of a more appropriate description)
a uniform distribution over a compact set $A\subset\R^\numFitParams$ (for instance $A=[a,b]^\numFitParams$ with $0<a<b$), hence
\begin{equation*}
f_{\mathbf P}(\mathbf p)=
\begin{cases}
\frac1{|A|}&\text{if }\mathbf p\in A\\
0&\text{ else,}
\end{cases}
\end{equation*}
where $|A|$ denotes the volume of $A$, and $\numFitParams=(2+m)\numAtoms+1$ for the general model or $\numFitParams=3\numAtoms+2$ for the Gaussian bump model with a single scalar unknown for the width of the bumps.

\subsection{Model for the intensity noise}\label{sec:noiseModel}
HAADF-STEM essentially counts the electrons from the beam that are deflected by the sample and arrive at a circular annulus centered around the sampling position.
Since the electron deflection events occur independently, this electron count must clearly be Poisson distributed, however,
it is typically multiplied by some gain factor $\alpha>0$ and perturbed by an additive Gaussian noise of distribution $\Gauss(\mu,\sigma^2)$ modeling the (device-specific) background noise of the sensor. While the noise model is usually not studied in detail in the microscopy literature, mixed Poisson-Gaussian noise is suitable in general for CMOS sensors~\cite{FoTrKa08} and thus also applicable to STEM.
Thus, given a fixed and constant material density $\den$, the measured (stochastic) STEM signal $\Img$ during a time interval of length $\Delta t$
is the sum of a Poisson and a Gaussian distributed random variable,
\begin{equation*}
\Img=\alpha\Img_P+\Img_G
\qquad\text{with }
\Img_P\sim\Poisson(\Delta t\den)
\text{ and }
\Img_G\sim\Gauss(\mu,\sigma^2)\,,
\end{equation*}
where $\Poisson(\lambda)$ denotes the Poisson distribution with mean $\lambda$.
Abbreviating $\Delta t\den=z$, the corresponding probability density function is thus given by
\begin{equation*}
f_{\Img}(\img)
=\sum_{k=0}^\infty\frac{z^ke^{-z}}{k!}\frac1{\sqrt{2\pi}\sigma}\exp\left(-\frac{(\img-\alpha k-\mu)^2}{2\sigma^2}\right)=:f^\data(\img;z)\,.
\end{equation*}
For very small $\sigma\ll\alpha z$ the Gaussian component is negligible unless $\alpha k\approx\img-\mu$,
while for very large $\sigma\gg\alpha z$ the Poisson component ensures a concentration at $k\approx z$ so that the distribution can be approximated by
\begin{equation*}
f^\data(\img;z)\approx\begin{cases}
\frac{z^{\llbracket\frac{\img-\mu}\alpha\rrbracket}e^{-z}}{\llbracket\frac{\img-\mu}\alpha\rrbracket!}&\text{if }\sigma\ll\alpha z\\
C\frac1{\sqrt{2\pi}\sigma}\exp\left(-\frac{(\img-\alpha z-\mu)^2}{2\sigma^2}\right)&\text{if }\sigma\gg\alpha z
\end{cases}
\end{equation*}
(where $\llbracket\cdot\rrbracket$ denotes rounding to the nearest nonnegative integer),
however, one can also stick with a numerical approximation of $f^\data(\img;z)$.
Note that above we integrated out the actual electron count $k$ and the actual realization $\img-\alpha k$ of the background noise, since neither quantity is of interest to us.
This is conceptually different from keeping one of them as an auxiliary variable as we have done for $\motion$.
The resulting maximum a posteriri estimate for $\den$ will differ slightly between both approaches,
where the approach without additional variables takes into account the stochastic behavior in a larger region of the probability space.

\subsection{Influence of sample motion on signal}
Unfortunately, during a measurement interval of time $\Delta t$ at a position $x_i$ the sample is not stationary,
but undergoes a stochastic translation according to two-dimensional Brownian motion.
Recall that the (continuous) path of the sample as a function of time $t\in\R$ is described by $w_t\in\R^2$.
Thus, the material density during the measurement interval $[t_i-\Delta t,t_i)$ at position $x_i$ is not $\den(x_i)$,
but rather changes over time and is given by $\den(x_i+w_t)$.
Consequently, the electron count is not distributed according to $\Poisson(\Delta t\den(x_i))$,
but rather according to
\begin{equation*}
\Img_P\sim\Poisson\left(\int_{t_i-\Delta t}^{t_i}\den(x_i+w_t)\,\d t\right)\,.
\end{equation*}
Indeed, partition the interval $[t_i-\Delta t,t_i)$ into $M$ subintervals $I_{1},\ldots,I_M$ of length $h=\frac{\Delta t}M$
and let $\underline\den_{ij}=\min_{t\in I_j}\den(x_i+w_t)$, $\overline\den_{ij}=\max_{t\in I_j}\den(x_i+w_t)$.
Then, the electron count $\Img_{Pj}$ during the $j$\textsuperscript{th} subinterval is distributed according to a distribution $\Poisson_j$
between $\underline\Poisson_j=\Poisson(h\underline\den_{ij})$ and $\overline\Poisson_j=\Poisson(h\overline\den_{ij})$ in the sense
that the cumulative distribution function of $\Poisson_j$ lies everywhere between those of $\underline\Poisson_j$ and $\overline\Poisson_j$.
Likewise, letting $\underline\Img_{Pj}$ and $\overline\Img_{Pj}$ denote random variables
with distribution $\underline\Poisson_j$ and $\overline\Poisson_j$, respectively,
the distribution of the sum $\Img_P=\sum_{j=1}^M\Img_{Pj}$ lies between the distributions of $\sum_{j=1}^M\underline\Img_{Pj}$ and $\sum_{j=1}^M\overline\Img_{Pj}$,
which are well-known to be Poisson distributed according to
\begin{equation*}
\sum_{j=1}^M\underline\Img_{Pj}\sim\Poisson\left(\sum_{j=1}^Mh\underline\den_{ij}\right)\,,\qquad
\sum_{j=1}^M\overline\Img_{Pj}\sim\Poisson\left(\sum_{j=1}^Mh\overline\den_{ij}\right)\,.
\end{equation*}
As $M\to\infty$ and $h\to0$, both distributions uniformly converge against $\Poisson\left(\int_{t_i-\Delta t}^{t_i}\den(x_i+w_t)\,\d t\right)$ so that summarizing,
the measured signal at location $x_i$ satisfies
\begin{equation*}
\Img_i=\alpha\Img_P+\Img_G
\qquad\text{with }
\Img_P\sim\Poisson\left(\int_{t_i-\Delta t}^{t_i}\den(x_i+w_t)\,\d t\right)
\text{ and }
\Img_G\sim\Gauss(\mu,\sigma^2)\,.
\end{equation*}

\subsection{Signal distribution for stochastic motion}
Since $w_t$ is not known, but just a realization of $W_{\diff t}$, the parameter $\int_{t_i-\Delta t}^{t_i}\den(x_i+w_t)\,\d t$ inside the distribution of $\Img_i$
is itself just a realization of the random variable $\Lambda=\int_{t_i-\Delta t}^{t_i}\den(x_i+W_{\diff t})\,\d t$.
Thus, the Poisson component $\Img_P$ of the signal is actually distributed according to a probability distribution
\begin{equation*}
P(k)=\int_0^\infty\frac{\lambda^ke^{-\lambda}}{k!}f_{\Lambda}(\lambda)\,\d\lambda\,,
\end{equation*}
where $f_\Lambda$ denotes the probability density of $\Lambda$, still to be determined.

Let us assume $\den$ is analytic (this is true for our model of $\den$) and thus can be expanded into a Taylor series about any position $x_i+w_i$,
\begin{equation*}
\den(x)=\sum_{j=0}^\infty u_{ij}(\underbrace{x-x_i-w_i,\ldots,x-x_i-w_i}_{j\text{ times}})
\qquad\text{with }
u_{ij}=\tfrac1{j!}D^j\den(x_i+w_i)\,.
\end{equation*}
As a consequence,
\begin{equation*}
\Lambda
=\Delta t\den(x_i+w_i)+\tilde\Lambda
\qquad\text{with }
\tilde\Lambda=\sum_{j=1}^\infty u_{ij}\left(\int_{t_i-\Delta t}^{t_i}W_{\diff t}\,\d t-\Delta tw_i,\ldots,\int_{t_i-\Delta t}^{t_i}W_{\diff t}\,\d t-\Delta tw_i\right)\,.
\end{equation*}
It is readily verified that as long as the $u_{ij}$ decrease fast enough in $j$ (for instance, if the derivatives $|D^j\den(x_i+w_i)|$ increase at most exponentially in $j$),
then $\tilde\Lambda$ is a random variable with moments (in particular mean and standard deviation) that are small compared to $\Delta t$
(for instance, we show further below that the highest order term in $\tilde\Lambda$ is of order $\Delta t(w_i-w_{i-1})$).
Since the dwell time $\Delta t$ is small, we shall in our model approximate $\Lambda\approx u_i:=\Delta t\den(x_i+w_i)$ and $f_\Lambda=\delta_{u_i}$,
where $\delta_{u_i}$ denotes the Dirac measure centered at $u_i$.
Thus, we obtain $P(k)=\frac{u_i^ke^{-u_i}}{k!}$ and
\begin{equation*}
\Img_i=\alpha\Img_P+\Img_G
\qquad\text{with }
\Img_P\sim\Poisson\left(u_i\right)
\text{ and }
\Img_G\sim\Gauss(\mu,\sigma^2)
\end{equation*}
and therefore
\begin{equation*}
f_{\Img_i}(\img_i\,|\,W_{\diff t_{i-1}}=w_{i-1},W_{\diff t_i}=w_i,\Den=\den)
=f^\data(\img_i;u_i)\,.
\end{equation*}

\paragraph{Improved approximation.}
If desired, the distribution of $\Lambda$ can be more accurately approximated, for instance by incorporating terms of higher order in $\Delta t$.
As above we have
\begin{equation*}
\Lambda
=\Delta t\den(x_i+w_i)+D\den(x_i+w_i)\left(\int_{t_i-\Delta t}^{t_i}W_{\diff t}\,\d t-\Delta tw_i\right)+\hat\Lambda\,,
\end{equation*}
where $\hat\Lambda$ is a random variable with negligible moments if the $u_{ij}$ decrease fast enough.
We now examine the distribution of $X_i=\int_{t_i-\Delta t}^{t_i}W_{\diff t}\,\d t$ under the conditions $W_{\diff t_{i-1}}=w_{i-1}$ and $W_{\diff t_i}=w_i$.
For the time being let us assume $t_{i-1}<t_i-\Delta t$ strictly.
Approximating the integral by a Riemann sum with interval width $h=\frac{\Delta t}M$ and $t_i^j=t_{i}-\Delta t+jh$ for $j=0,\ldots,M$ yields
\begin{multline*}
X_i^h=\sum_{j=1}^M(t_i^j-t_i^{j-1})W_{\diff t_i^j}
=(t_i^M-t_i^0)W_{\diff t_i^M}+\sum_{j=1}^{M-1}(t_i^j-t_i^0)(W_{\diff t_i^j}-W_{\diff t_i^{j+1}})\\
=\Delta tw_{i}+\sum_{j=1}^{M-1}(t_i^j-t_i^0)(W_{\diff t_i^j}-W_{\diff t_i^{j+1}})\,.
\end{multline*}
Setting $t_i^{-1}=t_{i-1}$, $\Delta T=t_i-\Delta t-t_{i-1}$ and $Y_j=W_{\diff t_i^j}-W_{\diff t_i^{j-1}}$ for $j=0,\ldots,M$
we first note that $\mathbf Y=(Y_0,\ldots,Y_M)^T$ is normally distributed with mean $0\in\R^{M+1}$ and diagonal variance $V=\diff\mathrm{diag}(\Delta T,h,\ldots,h)$, that is, $\mathbf Y\sim\Gauss(0,V)$.
Introducing the vectors $a=(0,0,t_i^1-t_i^0,\ldots,t_i^{M-1}-t_i^0)^T$ and $e=(1,\ldots,1)^T$ we obtain
\begin{equation*}
X_i^h
=\Delta tw_{i}-\mathbf Y^Ta
=\Delta tw_{i}-\tfrac{e^TVa}{e^TVe}\mathbf Y^Te-\mathbf Y^T(a-\tfrac{e^TVa}{e^TVe}e)
=\Delta tw_{i}+X_{i,1}^h+X_{i,2}^h\,.
\end{equation*}
Obviously, $X_{i,1}^h$ and $X_{i,2}^h$ are normally distributed.
Furthermore, since $e$ and $(a-\tfrac{e^TVa}{e^TVe}e)$ are $V$-orthogonal, $X_{i,1}^h$ and $X_{i,2}^h$ are independent.
Using $e^TVa=\diff h^2\frac{M(M-1)}2=\diff\frac{M-1}{2M}\Delta t^2$ and $e^TVe=\diff(Mh+\Delta T)=\diff(\Delta t+\Delta T)$ we have
\begin{align*}
X_{i,1}^h&\sim\Gauss\left(0,\left[(\tfrac{M-1}{M}\tfrac{\Delta t}{\Delta t+\Delta T}\tfrac{\Delta t}2)^2\left[\Delta T+Mh\right]\right]\diff I_2\right)
=\Gauss\left(0,\left[(\tfrac{M-1}{M})^2(\tfrac{\Delta t}2)^2\tfrac{\Delta t^2}{\Delta t+\Delta T}\right]\diff I_2\right)\,,\\
X_{i,2}^h&\sim\Gauss\left(0,\left[(\Delta T+h)(\tfrac{M-1}{M}\tfrac{\Delta t}{\Delta t+\Delta T}\tfrac{\Delta t}2)^2+h\sum_{j=2}^M(t_i^{j-1}-t_i^0-\tfrac{M-1}{M}\tfrac{\Delta t}{\Delta t+\Delta T}\tfrac{\Delta t}2)^2\right]\diff I_2\right)\,.
\end{align*}
Now the condition $W_{t_i}=w_i$ and $W_{t_{i-1}}=w_{i-1}$ is equivalent to $\tfrac{M-1}{M}\tfrac{\Delta t}{\Delta t+\Delta T}\tfrac{\Delta t}2(w_i-w_{i-1})=\tfrac{e^TVa}{e^TVe}\sum_{j=0}^MY_j=-X_{i,1}^h$,
thus under this condition we have
\begin{multline*}
X_i^h
=\Delta tw_{i}-\tfrac{M-1}{M}\tfrac{\Delta t}{\Delta t+\Delta T}\tfrac{\Delta t}2(w_i-w_{i-1})+X_{i,2}^h
\sim\Gauss\Bigg(\Delta tw_{i}-\tfrac{M-1}{M}\tfrac{\Delta t}{\Delta t+\Delta T}\tfrac{\Delta t}2(w_i-w_{i-1}),\\
\textstyle\left[(\Delta T+h)(\tfrac{M-1}{M}\tfrac{\Delta t}{\Delta t+\Delta T}\tfrac{\Delta t}2)^2+h\sum_{j=2}^M(t_i^{j-1}-t_i^0-\tfrac{M-1}{M}\tfrac{\Delta t}{\Delta t+\Delta T}\tfrac{\Delta t}2)^2\right]\diff I_2\Bigg)\,.
\end{multline*}
As $h=\frac1M\to0$ this converges against
\begin{multline*}
X_i
\sim\Gauss\left((\tfrac{\Delta T+\Delta t/2}{\Delta T+\Delta t}w_i+\tfrac{\Delta t/2}{\Delta T+\Delta t}w_{i-1})\Delta t,\left[(\tfrac{\Delta t/2}{\Delta T+\Delta t})^2\Delta T\Delta t^2+\int_{t_i-\Delta t}^{t_i}(t-t_i^0-\tfrac{\Delta t/2}{\Delta T+\Delta t}\Delta t)^2\,\d t\right]\diff I_2\right)\\
=\Gauss\left((\tfrac{\Delta T+\Delta t/2}{\Delta T+\Delta t}w_i+\tfrac{\Delta t/2}{\Delta T+\Delta t}w_{i-1})\Delta t,\left[(\tfrac{\Delta t/2}{\Delta T+\Delta t})^2\Delta T\Delta t^2+\tfrac{\Delta t^3}3\left[(\tfrac{\Delta T+\Delta t/2}{\Delta T+\Delta t})^3+(\tfrac{\Delta t/2}{\Delta T+\Delta t})^3\right]\right]\diff I_2\right)\,.
\end{multline*}
In the case $\Delta T=0$, an analogous argument leads to
\begin{equation*}
X_i\sim\Gauss(\Delta t\tfrac{w_i+w_{i-1}}2,\tfrac{\Delta t^3}{12}\diff I_2)\,.
\end{equation*}
As a result,
\begin{equation*}
\Lambda-\hat\Lambda\sim\Gauss\left(m,s^2\right)
\end{equation*}
for the mean and variance
\begin{align*}
m&=\Delta t\left[\den(x_i+w_i)-\tfrac{\Delta t/2}{\Delta T+\Delta t}D\den(x_i+w_i)(w_i-w_{i-1})\right]\,,\\
s^2&=\diff\left[(\tfrac{\Delta t/2}{\Delta T+\Delta t})^2\Delta T\Delta t^2+\tfrac{\Delta t^3}3\left[(\tfrac{\Delta T+\Delta t/2}{\Delta T+\Delta t})^3+(\tfrac{\Delta t/2}{\Delta T+\Delta t})^3\right]\right]|D\den(x_i+w_i)|^2\,.
\end{align*}
With this approximation for $\Lambda$ we obtain
\begin{equation*}
f_{\Img_i}(\img_i\,|\,W_{\diff t_{i-1}}=w_{i-1},W_{\diff t_i}=w_i,\Den=\den)
=\int_\R f^\data(\img_i;z)\tfrac1{\sqrt{2\pi}s}\exp(-\tfrac{(z-m)^2}{2s^2})\,\d z\,,
\end{equation*}
which only depends on $\img_i,w_{i-1}$, $w_i$, $\den(x_i+w_i)$, $D\den(x_i+w_i)$, $\mu$, $\sigma$, and $\alpha$.

\subsection{The variational model}\label{sec:variationalModel}
Since $\Img_i$ only depends on $\Den$, $W_{\diff t_{i-1}}$, and $W_{\diff t_i}$, we may write
\begin{equation*}
f_{\mathbf\Img}(\mathbf\img\,|\,\Motion=\motion,\Den=\den)
=\prod_{i=1}^Nf_{\Img_i}(\img_i\,|\,W_{\diff t_{i-1}}=w_{i-1},W_{\diff t_i}=w_i,\Den=\den)\,.
\end{equation*}
As explained previously, we shall be looking for the widely used maximum a posteriori (MAP) estimate of $\motion$ and $\den$ (or equivalently $\mathbf p$).
To this end, we shall minimize the negative logarithm of $f_{\Motion,\mathbf P}(\motion,\mathbf p\,|\,\mathbf\Img=\mathbf\img)$, which by Section\,\ref{sec:Bayes} and the subsequent sections can be expressed as
\begin{align*}
E[\motion,\mathbf p]
&=-\log f_{\Motion,\mathbf P}(\motion,\mathbf p\,|\,\mathbf\Img=\mathbf\img)\\
&=\text{const.}-\sum_{i=1}^N\log f_{\Img_i}(\img_i\,|\,W_{\diff t_{i-1}}=w_{i-1},W_{\diff t_i}=w_i,\Den=\den[\mathbf p])-\log f_{\Motion}(\motion)-\log f_{\mathbf P}(\mathbf p)\\
&=\text{const.}+\sum_{i=1}^N\diss(\img_i,\Delta t\den[\mathbf p](x_i+w_i))+\frac1{2\diff}\sum_{i=1}^N\frac{|w_i-w_{i-1}|^2}{(t_i-t_{i-1})}+\iota_A(\mathbf p)\,,
\end{align*}
where $\iota_A$ denotes the indicator function of the set $A$ and
\begin{equation*}
\diss(g,z)=-\log f^\data(g;z)
\end{equation*}
is a data dissimilarity.
The minimizers $(\motion,\mathbf p)$ of the energy $E$ serve as estimates of the true sample motion and atom configuration,
where the well-posedness of the minimization is shown in the following theorem.
\begin{theorem}[Existence of minimizers]\label{thm:existence}
Let $\mathbf p\mapsto\den[\mathbf p]$ be a continuous mapping from $A$ into $C(\bar\Omega)$.
Then, $E[\motion,\mathbf p]$ possesses a minimizer in $\R^{2N}\times A$.
\end{theorem}
\begin{proof}
The energy is lower semi-continuous in all variables (note in particular that due to the continuous dependence of $\den[\mathbf p]$ on $\mathbf p$ the evaluation at $x_i+w_i$ is continuous).
Furthermore, the enery is coercive in the sense that if any component of the variables diverges, then so does the energy.
Finally, the choice $\motion=0$ and $\mathbf p\in A$ arbitrary yields finite energy, and all energy terms are globally bounded from below by a constant only depending on $\mathbf\img$.
Existence of minimizers thus follows by the standard direct method of the calculus of variations.
\end{proof}

For later reference, let us here also state a property of the data dissimilarity $\diss$.
The property essentially means that for small enough Gaussian sensor noise, $\diss$ approximates the Kullback--Leibler divergence.
This is not surprising, since for small $\sigma$ the measured signals will almost follow the Poisson distribution, whose logarithm is well-known to lead to the Kullback--Leibler divergence.
In the following, $o$ and $O$ denote the Landau symbols.
Also recall the notation $\llbracket\cdot\rrbracket$ for rounding to the nearest \emph{nonnegative} integer.
\begin{lemma}\label{thm:dataTerm}
Let $\frac\sigma\alpha\to0$.
Furthermore, let $\frac{g-\mu}\alpha\in o(\exp((\frac\alpha\sigma)^2))$, $\left|\llbracket\tfrac{g-\mu}\alpha\rrbracket-\tfrac{g-\mu}\alpha\right|<c<\frac12$ and $z\in[\underline z,\overline z]$ for some fixed $0<\underline z<\overline z$,
then $$\diss(g,z)=z-\llbracket\tfrac{g-\mu}\alpha\rrbracket\log z+\tilde C(g-\mu,\alpha,\sigma)+O(\tfrac\sigma\alpha)\,,$$
where the constant $\tilde C(g-\mu,\alpha,\sigma)$ only depends on $g-\mu$, $\alpha$, and $\sigma$, but not on $z$.
\end{lemma}
\begin{proof}
Let us abbreviate $K=\llbracket\tfrac{g-\mu}\alpha\rrbracket$, then $\log K\in o(\frac{\alpha^2}{\sigma^2})$.
We have
\begin{equation*}
f^\data(g;z)
=\frac1{\sqrt{2\pi}\sigma}\frac{e^{-z}z^K}{K!}\exp\left(-\frac{(g-\mu-\alpha K)^2}{2\sigma^2}\right)\left(1+\sum_{\substack{k=0\\k\neq K}}^\infty r_k\right)
\end{equation*}
for
\begin{align*}
r_k
&=\frac{K!}{k!}z^{k-K}\exp\left(\frac{(g-\mu-\alpha K)^2-(g-\mu-\alpha k)^2}{2\sigma^2}\right)\\
&=\exp\left(\log K!-\log k!+(k-K)\log z+\frac{\alpha^2}{2\sigma^2}\left[\left(\frac{g-\mu}\alpha-K\right)^2-\left(\frac{g-\mu}\alpha-k\right)^2\right]\right)\\
&\leq\exp\left(\log K!-\log k!+|K-k|C+\frac{\alpha^2}{2\sigma^2}\left[c^2-(1-c)^2\left(K-k\right)^2\right]\right)\\
&\leq\exp\left(\log K!-\log k!+|K-k|C-\frac{\alpha^2}{2\sigma^2}[1-2c]\left(K-k\right)^2\right)
\end{align*}
for $C=\max(|\log\underline z|,|\log\overline z|)$,
where we used $(\frac{g-\mu}\alpha-k)^2=\left(1+\frac{\frac{g-\mu}\alpha-K}{K-k}\right)^2(K-k)^2\geq(1-\frac c1)^2(K-k)^2$.
Now, if $k>K$, then
\begin{equation*}
r_k
\leq\exp\left(C(k-K)-\tfrac{\alpha^2}{2\sigma^2}\left[1-2c\right]\left(K-k\right)^2\right)
\leq\exp\left(\left[C-\tfrac{\alpha^2}{2\sigma^2}[1-2c]\right]\left(K-k\right)^2\right)\,,
\end{equation*}
while for $k<K$,
\begin{multline*}
r_k
\leq\exp\left(\log K+\log(K-1)+\ldots+\log(k+1)+C(K-k)-\tfrac{\alpha^2}{2\sigma^2}\left[1-2c\right]\left(K-k\right)^2\right)\\
\leq\exp\left((\log K+C)(K-k)-\tfrac{\alpha^2}{2\sigma^2}\left[1-2c\right]\left(K-k\right)^2\right)
\leq\exp\left(\left[C+\log K-\tfrac{\alpha^2}{2\sigma^2}[1-2c]\right]\left(K-k\right)^2\right)\,.
\end{multline*}
Thus, for $\frac\alpha\sigma$ large enough,
we obtain in both cases $r_k\leq\exp\left(-\tfrac{\alpha^2}{4\sigma^2}[1-2c]\left(K-k\right)^2\right)$ so that
\begin{equation*}
0
\leq\sum_{\substack{k=0\\k\neq K}}^\infty r_k
\leq\sum_{\substack{k=0\\k\neq K}}^\infty\exp\left(-\frac{\alpha^2}{4\sigma^2}[1-2c]\left(K-k\right)^2\right)
\leq\int_{-\infty}^\infty\exp\left(-\frac{\alpha^2}{4\sigma^2}[1-2c]\left(K-k\right)^2\right)\,\d k
=2\sqrt{\frac\pi{1-2c}}\frac\sigma\alpha\,.
\end{equation*}
Therefore,
\begin{equation*}
\diss(g,z)
=-\log f^\data(g;z)
=z-K\log z+\tfrac{\log(2\pi\sigma^2)}2+\log K!+\frac{(g-\mu-\alpha K)^2}{2\sigma^2}-\log\left(1+O\left(\frac\sigma\alpha\right)\right)\,.
\qedhere
\end{equation*}
\end{proof}

\subsection{Scan mode and redundant image acquisition}\label{sec:scanMode}
Here, we specify the scan mode with which the experimental data of the later sections is obtained.
In those experiments, the instrument measures pixels row by row along $N_2$ rows of $N_1$ pixels each, that is,
there are horizontal and vertical pixel distances $\Delta x_1,\Delta x_2$ with $N_1\Delta x_1=N_2\Delta x_2=a$, $N=N_1N_2$, and
\begin{equation}
\label{eq:STEMIndexing}
\left.
\begin{split}
x_i&=x_{ml}=(m\Delta x_1,l\Delta x_2)\,,\\
t_i&=(lN_1+m)\Delta t+l\Delta T
\end{split}
\right\}
\quad\text{for }
m=(i-1)\mathrm{mod}N_1+1
\text{ and }
l=\lfloor\tfrac{i-1}{N_1}\rfloor+1\,,
\end{equation}
$\lfloor\cdot\rfloor$ representing the integer part.

Finally, to obtain a better resolution, the same sample is sometimes imaged multiple times, say $K$ times, potentially after sample rotations $R_k\in SO(2)$, $k=1,\ldots,K$, yielding measurements $\mathbf{\img}^k$, $k=1,\ldots,K$.
In that case, we have Brownian motions $\motion^k$ for each image acquisition, and an analogous derivation as before yields the objective functional
\begin{equation}\label{eqn:objectiveMultipleInputs}
E^K[\motion^1,\ldots,\motion^K,\mathbf p]
=\sum_{k=1}^K\left[\sum_{i=1}^N\diss(\img_i^k,\Delta t\den[\mathbf p](R_k(x_i+w_i^k)))+\frac1{2\diff}\sum_{i=1}^N\frac{|w_i^k-w_{i-1}^k|^2}{(t_i-t_{i-1})}\right]+\iota_A(\mathbf p)\,,
\end{equation}
to be minimized for $\motion^1,\ldots,\motion^K,\mathbf p$,
where as usual, we abbreviated $w_0^k=0$.
An analogous argument as in Theorem~\ref{thm:existence} shows the well-posedness of this slightly extended energy.

\section{The corresponding continuous limit models}\label{sec:limitModels}
STEM essentially samples the probe at a discrete set of points.
This may be viewed as a discretization of a continuous measurement,
and one may ask the question what happens as measurements are acquired at more and more locations,
thereby increasing the scan resolution until in the limit the instrument measures the sample in a continuous manner.
Of course, in parallel, the dwell time $\Delta t$ has to decrease at the same rate in order to keep the total image acquisition time bounded.
The examination of this limit process may give some hints as to how to choose the scanning parameters.
For ease of exposition, we just consider the case of a single image acquisition; the case of $K$ acquisitions follows in exactly the same way.

\subsection{Limit model for measurements via stochastic coupling}
In order to study the limit as the resolution of the scanning path and the number of measurement locations tend to infinity,
we first have to comprehend how the measurement vector $\mathbf\img$ behaves for finer and finer resolution.
Note, for instance, that the dimension of $\mathbf\img$ will increase with the number of measurement locations.
Unfortunately, $\mathbf\img$ cannot be determined deterministically for a given resolution, since it is just one possible realization of a random variable.
Therefore, we must rather understand how the distribution of $\mathbf\img$ varies as the resolution gets finer.
This is possible using standard models of Poisson point processes.

Consider a time-continuous measurement (in our case the electron count) during the time interval $[0,T]$
(for instance the total time interval to acquire a full STEM image).
The local material density underneath the electron beam at time $t\in[0,T]$ shall be $\lambda(t)$
(in our case, $\lambda(t)=\den(x(t)+w(t))$, where $x$ denotes the desired beam position and $w$ the Brownian motion of the sample).
The events of an electron being detected can mathematically be formulated as a measure
\begin{equation*}
\elec=\sum_{j=1}^{N_\elec}\delta_{\tau_j}\in\meas([0,T])\,,
\end{equation*}
where each Dirac mass $\delta_{\tau_j}$ describes the detection of an eletron at time $\tau_j\in[0,T]$ and $\meas([0,T])$ denotes the space of Radon measures on $[0,T]$.
The times $\tau_j$ are distributed over $[0,T]$ according to an inhomogeneous Poisson point process with intensity $\lambda(t)$
(for an introduction to Poisson processes, their existence, and their interpretation as empirical processes as exploited here see for instance \cite[Sec.\,2.1 and 2.5]{Ki93}, \cite[Sec.\,1.2]{Re93}),
thus each $\elec$ is just a realization of a random variable $\Elec$ representing the Poisson point process.
The actual measurement $\img_i$ is then obtained as
\begin{equation*}
\img_i=\alpha\elec([t_{i}-\Delta t,t_i))+\img_G\,,
\end{equation*}
where $\img_G$ is the realization of the sensor noise $\Img_G$ at measurement position $i$.
Note that in our model derivation we only had to exploit the fact that $\img_P=\elec([t_{i}-\Delta t,t_i))$ is a realization of a Poisson distributed random variable, but we did not have to resolve the single electron counting events temporally.

Now consider a particular scan of a particular sample with material density $\den$.
We perform the thought experiment that we acquire measurements, indexed by superscript $n$, of this particular physical situation
(including the fixed realization $w$ of the Brownian motion) at finer and finer resolution,
that is, with a dwell time $\Delta t^n\to0$ as $n\to\infty$ and the number of measurement locations $N^n\to\infty$.
Note that the corresponding density $\lambda^n(t)$ changes as $n\to\infty$, since the scanning path changes slightly.
Consequently, also the distribution $\Elec^n$ of electron detections changes.

The change of this distribution can be understood via the following standard stochastic coupling (an extension coupling in the terminology of \cite[Chp.\,3, Sec.\,3.1]{Th00}). Let
\begin{equation*}\textstyle
B=\left\{\hat\elec=\sum_{i=1}^\infty\delta_{\tau_i,\lambda_i}\,:\,\tau_i\in[0,T],\,\lambda_i\in[0,\infty)\,\forall i,\,\hat\elec([0,T]\times[0,\lambda])<\infty\,\forall\lambda\geq0\right\}\,.
\end{equation*}
On $B$ we impose the natural $\sigma$-algebra $\mathcal F$
which is generated by the maps $\hat\elec\mapsto\hat\elec(S)$ for all Borel measurable $S\subset [0,T]\times[0,\infty)$.
Finally, let $\elecP$ be the probability measure on $B$
such that the support $\spt\hat\elec$ of $\hat\elec$ is distributed according to a Poisson point process on $[0,T]\times[0,\infty)$ with intensity $1$.
For given $\hat\elec\in B$ and $\lambda:[0,T]\to[0,\infty)$ we now introduce
\begin{equation*}
\hat\elec_\lambda=\sum_{\substack{(\tau_j,\lambda_j)\in\spt\hat\elec\\\lambda_j\leq\lambda(\tau_j)}}\delta_{\tau_j}\,.
\end{equation*}
It is known that $\hat\elec_\lambda$ is distributed according to a Poisson point process with intensity $\lambda$
(this is a direct consequence of the mapping theorem \cite[p.\,18 and example\,(2.30)-(2.31)]{Ki93}, \cite[Lem.\,1.1.3]{Re93})
and therefore exactly like the electron count $\elec$.
Thus, the above thought experiment may be performed at fixed realization $\hat\elec$ (setting $\elec^n=\hat\elec_{\lambda^n}$),
which describes appropriately how the physical measurements or their distributions change as $n\to\infty$.

We also have to specify how to deal with the second component of the signal, the background sensor noise, as $n\to\infty$.
For reasons to become clear later, we assume the variance $\sigma^2$ of the Gaussian background noise to change with $n$.
Let $\White$ be independent Gaussian white noise on $[0,T]$, and denote the times at which the electron beam moves from one to the next measurement location by $t_1^n,\ldots,t_{N^n}^n$, where the superscript $n$ refers to the sequence of thought experiments with increasing resolution.
For a realization $\white$ of $\White$ define
\begin{equation*}
\mathbf\white^n=(\white(t_1^n),\ldots,\white(t_{N^n}^n))\,,
\end{equation*}
then $\mathbf\img_G^n=\mu+\sigma^n\mathbf\white^n$ is distributed in the same way as the background noise of the measurement.
Therefore we may perform the above thought experiment for fixed realization $\white$.

The following lemma now analyses how the signal in our thought experiment behaves as the material density $\lambda^n(t)$ under the electron beam changes for $n\to\infty$.
\begin{lemma}\label{thm:dataBehaviour}
For given fixed $\hat\elec\in B$ and white noise realization $r$, define the electron counting events, the electron counts, the background noise, and the full measurement as
\begin{align*}
\elec^n&=\hat\elec_{\lambda^n}\,,\\
\mathbf\img_P^n&=(\elec^n([t_0^n,t_1^n)),\ldots,\elec^n([t_{N^n-1}^n,t_{N^n}^n)))\,,\\
\mathbf\img_G^n&=\mu+\sigma^n(\white(t_1^n),\ldots,\white(t_{N^n}^n))\,,\\
\mathbf\img^n&=\alpha\mathbf\img_P^n+\mathbf\img_G^n\,,
\end{align*}
where $0=t_0^n<\ldots<t_{N^n}^n=T$.
We further introduce continuum versions of the signals as
\begin{equation*}
\mathcal G_P^n(t)=\frac{(\mathbf\img_P^n)_i}{t_{i}^n-t_{i-1}^n}\,,\qquad
\mathcal G_G^n(t)=\frac{(\mathbf\img_G^n)_i-\mu}{\alpha(t_{i}^n-t_{i-1}^n)}\,,\qquad
\mathcal G^n(t)=\frac{(\mathbf\img^n)_i-\mu}{\alpha(t_{i}^n-t_{i-1}^n)}\,,\qquad\text{each for }t\in[t_{i-1}^n,t_i^n)\,.
\end{equation*}
Finally, we shall assume $\sup_{i=1,\ldots,N^n}(t_i^n-t_{i-1}^n)\to0$ as $n\to\infty$. The following statements hold true almost surely with respect to the distribution of $\hat\elec$ and $r$.
\begin{enumerate}
\item\label{enm:elecCountConv} If $\lambda^n\to\lambda^\infty$ uniformly, then $\elec^n\to\elec^\infty$ strongly in $\meas([0,T])$.
\item\label{enm:elecCountWeakConv} If $\lambda^n\stackrel*\rightharpoonup\lambda^\infty$ in $L^\infty((0,T))$, then there exists a subsequence with $\elec^n\to\elec^\infty$ strongly in $\meas([0,T])$.
If $\lambda^n\not\to\lambda^\infty$ strongly in $L^1((0,T))$, then with positive probability $\elec^n\not\to\elec^\infty$ for the entire sequence.
\item\label{enm:backgNoiseConv} We have
\begin{equation*}
\|\mathcal G_G^n\|_{\meas([0,T])}\to\begin{cases}
0&\text{if }\frac{\sigma^nN^n}{\alpha}\to0\,,\\
\infty&\text{if }\frac{\sigma^nN^n}{\alpha}\to\infty\,,
\end{cases}
\end{equation*}
and $\mathcal G_G^n\stackrel*\rightharpoonup0$ in ${\meas([0,T])}$ if $\lim_{n\to\infty}\frac{\sigma^nN^n}{\alpha}\in(0,\infty)$.
\item\label{enm:niceData}
If $\lim_{n\to\infty}\frac{\sigma^nN^n}{\alpha}<\infty$ and $\lambda^n\to\lambda^\infty$ uniformly on $[0,T]$,
then $\mathcal G^n\stackrel*\rightharpoonup\elec^\infty$ in $\meas([0,T])$.
If only $\lambda^n\stackrel*\rightharpoonup\lambda^\infty$ in $L^\infty((0,T))$, then $\mathcal G^n\stackrel*\rightharpoonup\elec^\infty$ only for a subsequence.
\item\label{enm:cleanData}
If $\lim_{n\to\infty}\frac{\sigma^nN^n}{\alpha}=0$, letting $\tilde{\mathcal G}^n(t)=\frac1{t_i^n-t_{i-1}^n}\left(\llbracket(t_i^n-t_{i-1}^n)\mathcal G^n(t)\rrbracket-(t_i^n-t_{i-1}^n)\mathcal G^n(t)\right)$ for $t\in[t_{i-1}^n,t_i^n)$,
we have $\|\tilde{\mathcal G}^n\|_{\meas([0,T])}\to0$.
\end{enumerate}
Note that whenever we use a function, such as $\mathcal G_G^n$, in the sense of a measure, we refer to the measure induced when interpreting the function as density with respect to the Lebesgue measure.
\end{lemma}
\begin{proof}
\begin{enumerate}
\item Let $\varepsilon_n=\sup_{t\in[0,T]}|\lambda^n(t)-\lambda^\infty(t)|$,
then $\|\elec^n-\elec^\infty\|_{\meas([0,T])}\leq\hat\elec(S^n)$
for $S^n=\{(t,\lambda)\in[0,T]\times[0,\infty)\,:\,\lambda^\infty(t)-\varepsilon_n\leq\lambda\leq\lambda^\infty(t)+\varepsilon_n\}$.
Thus
\begin{multline*}
\elecP\big(\big\{\hat\elec\in B\,:\,\lim_{n\to\infty}\|\elec^n-\elec^\infty\|_{\meas([0,T])}=0\big\}\big)
\geq\elecP\big(\big\{\hat\elec\in B\,:\,\lim_{n\to\infty}\hat\elec(S^n)=0\big\}\big)\\
\geq\elecP\big(\big\{\hat\elec\in B\,:\,\hat\elec(S^n)=0\big\}\big)
=e^{-|S^n|}
\mathop{\to}_{n\to\infty}1\,,
\end{multline*}
where $|S^n|$ denotes the volume of $S^n$.
\item Abbreviate $M=\sup_{n=0,1,\ldots}\|\lambda^n\|_{L^\infty((0,T))}<\infty$
and denote by $V\subset B$ the set of Poisson processes $\hat\elec$ such that $\elec^n\not\to\elec^\infty$ for any subsequence.
Let $\hat\elec\in V$ with $m$ points $(\tau_1,\lambda_1),\ldots,(\tau_m,\lambda_m)$ in $[0,T]\times[0,M]$
and introduce $\zeta^{\hat\elec}=(\sign(\lambda_1-\lambda^\infty(\tau_1)),\ldots,\sign(\lambda_m-\lambda^\infty(\tau_m)))$.
Obviously, $\hat\elec\in V$ implies that for every large enough $n$ there is a point $(\tau_i,\lambda_i)\in\spt\hat\elec\subset[0,T]\times[0,\infty)$
such that $\lambda_i\in(\lambda^n(\tau_i),\lambda^\infty(\tau_i)]\text{ if }\zeta^{\hat\elec}_i<0\text{ and }\lambda_i\in(\lambda^\infty(\tau_i),\lambda^n(\tau_i)]\text{ else}$.
Choosing $\delta=\frac12\min(|\lambda_1-\lambda^\infty(\tau_1)|,\ldots,|\lambda_m-\lambda^\infty(\tau_m)|)$ we thus have $(\tau_1,\ldots,\tau_m)\in S(\lambda_1,\ldots,\lambda_m,\zeta^{\hat\elec})\subset S_\delta^m(\zeta^{\hat\elec})$ for
\begin{align*}
S(\lambda_1,\ldots,\lambda_m,\zeta^{\hat\elec})
&=\{(\tau_1,\ldots,\tau_m)\in[0,T]^m\,:\,\exists N>0\,\forall n>N\,\exists i\in\{1,\ldots,m\}:\\&\qquad\qquad\lambda_i\in(\lambda^n(\tau_i),\lambda^\infty(\tau_i)]\text{ if }\zeta^{\hat\elec}_i<0\text{ and }\lambda_i\in(\lambda^\infty(\tau_i),\lambda^n(\tau_i)]\text{ else}\}\,,\\
S_\delta^m(\zeta^{\hat\elec})
&=\{(\tau_1,\ldots,\tau_m)\in[0,T]^m\,:\,\exists N>0\,\forall n>N\,\exists i\in\{1,\ldots,m\}:\zeta^{\hat\elec}_i(\lambda^n(\tau_i)-\lambda^\infty(\tau_i))\geq\delta\}\,.
\end{align*}
Therefore, we obtain
\begin{align*}
V&\subset\bigcup_{m=1}^\infty\{\hat\elec\in B\,:\,\hat\elec([0,T]\times[0,M])=m,\,\spt(\hat\elec)=\{(\tau_1,\lambda_1),\ldots,(\tau_m,\lambda_m)\},\,(\tau_1,\ldots,\tau_m)\in S(\lambda_1,\ldots,\lambda_m,\zeta^{\hat\elec})\}\\
&\subset\bigcup_{m=1}^\infty\bigg\{\hat\elec\in B\,:\,\spt(\hat\elec)\cap[0,T]\times[0,M]=\{(\tau_1,\lambda_1),\ldots,(\tau_m,\lambda_m)\},\,(\tau_1,\ldots,\tau_m)\in\bigcup_{\zeta\in\{-1,1\}^m}\bigcup_{\delta>0}S_\delta^m(\zeta)\bigg\}\\
&\subset\bigcup_{m=1}^\infty\bigcup_{\zeta\in\{-1,1\}^m}\bigcup_{\delta>0}\{\hat\elec\in B\,:\,\spt(\hat\elec)\cap[0,T]\times[0,M]=\{(\tau_1,\lambda_1),\ldots,(\tau_m,\lambda_m)\},\,(\tau_1,\ldots,\tau_m)\in S_\delta^m(\zeta)\}\\
&=\bigcup_{m=1}^\infty\bigcup_{\zeta\in\{-1,1\}^m}\bigcup_{\delta>0}V_\delta^m(\zeta)\,.
\end{align*}
Thus, we get
\begin{equation*}
\elecP(V)
\leq\elecP\left(\bigcup_{l=1}^\infty\bigcup_{\zeta\in\{-1,1\}^m}\bigcup_{i\in\N}V_\delta^m(\zeta)\right).
\end{equation*}
Since $V_\delta^m(\zeta)$ is monotone in $\delta$, that is, $V_\delta^m(\zeta)\subset V_{\hat\delta}^m(\zeta)$ for any $\hat\delta<\delta$, we can replace the union over $\delta$ with a countably infinite union.
Hence, if we can show $\elecP(V_\delta^m(\zeta))=0$ for all $\delta>0$, $\zeta\in\{-1,1\}^m$, and $m\geq0$, we get $\elecP(V)=0$, proving the desired claim.

It remains to show $\elecP(V_\delta^m(\zeta))=0$ for fixed $\delta>0$, $\zeta\in\{-1,1\}^m$, $m\geq0$. To this end, we
\begin{enumerate}
\item\label{enm:nullset} show that $S_\delta^m(\zeta)$ is a nullset and
\item\label{enm:nullset2} imply that $\elecP(V_\delta^m(\zeta))=0$.
\end{enumerate}

As for the first item,
note that $S_\delta^m(\zeta)$ is measurable since it is the countable union of countable intersections of measurable sets,
\begin{multline*}
S_\delta^m(\zeta)
=\bigcup_{N=1}^\infty\bigcap_{n=N}^\infty\big[
\{\tau\in[0,T]\,:\zeta_1(\lambda^n(\tau)-\lambda^\infty(\tau))\geq\delta\}\times\R^{m-1}\,\cup\\
\R\times\{\tau\in[0,T]\,:\zeta_2(\lambda^n(\tau)-\lambda^\infty(\tau))\geq\delta\}\times\R^{m-2}\,\cup
\ldots\cup\\
\R^{m-1}\times\{\tau\in[0,T]\,:\zeta_m(\lambda^n(\tau)-\lambda^\infty(\tau))\geq\delta\}\big]\,.
\end{multline*}
Furthermore, the set
\begin{equation*}
\tilde S_\delta^m=S_\delta^m(\zeta)\setminus(\R\times S_\delta^{m-1}((\zeta_2,\ldots,\zeta_{m})))
\end{equation*}
is a nullset, which follows from \cite[Thm.\,B]{Do89} and the fact that $\tilde S_\delta^m$ is measurable and
that for any $\vec \tau\in[0,T]^{m-1}$ the set $\hat S_\delta^m=\{(\tau_1,\ldots,\tau_m)\in\tilde S_\delta^m\,:\,(\tau_2,\ldots,\tau_{m})=\vec \tau\}$ is an $\mathcal L^1$-nullset.
Indeed, for $\vec\tau\in S_\delta^{m-1}((\zeta_2,\ldots,\zeta_{m}))$ we have $\hat S_\delta^m=\emptyset$,
while for $\vec \tau\notin S_\delta^{m-1}((\zeta_2,\ldots,\zeta_{m}))$ there is an infinite subsequence $n_1,n_2,\ldots$ with $\zeta_i(\lambda_{n_j}(\vec \tau_i)-\lambda^\infty(\vec \tau_i))<\delta$ for all $i=2,\ldots,m$ and $j\in\N$ so that
\begin{multline*}
\hat S_\delta^m
=\{(\tau_1,\ldots,\tau_m)\in S_\delta^m(\zeta)\,:\,(\tau_2,\ldots,\tau_{m})=\vec \tau\}\\
=\{\tau\in[0,T]\,:\,\zeta_1(\lambda_{n_j}(\tau)-\lambda^\infty(\tau)\geq\delta\text{ for all }j\}\times\{\vec\tau\}
=\bar S_\delta^1(\zeta_1)\times\{\vec\tau\}\,.
\end{multline*}
Now the set $\bar S_\delta^1(\zeta_1)$ must be a nullset
since otherwise $\lambda_{n_j}-\lambda^\infty$ tested with the characteristic function of $\bar S_\delta^1(\zeta_1)$ would yield $\int_0^T\zeta_1\chi_{\bar S_\delta^1(\zeta_1)}(\lambda_{n_j}-\lambda^\infty)\,\d\tau\geq\delta|\bar S_\delta^1(\zeta_1)|$ for all $j$,
contradicting the weak-* convergence $\lambda^n\stackrel*\rightharpoonup\lambda^\infty$
(here $|\cdot|$ denotes the Lebesgue measure).
Finally, $S_\delta^1(\pm1)=\{t\in[0,T]\,:\,\exists N\geq0\,\forall n>N:\pm(\lambda^n(t)-\lambda^\infty(t))\geq\delta\}$ is a nullset for the same reason.
By induction in $m$ it thus follows that $S_\delta^m(\zeta)$ is a nullset.

Now let us derive $\elecP(V_\delta^m(\zeta))=0$.
For $\hat\elec\in B$ with $\hat\elec([0,T]\times[0,M])=m$ denote the points in its support by $\spt(\hat\elec)\cap[0,T]\times[0,M]=\{(\tau_1,\lambda_1),\ldots,(\tau_m,\lambda_m)\}$.
Obviously, $\tau_1,\ldots,\tau_m\in[0,T]$ are independently identically distributed random variables obeying the uniform distribution on $[0,T]$.
Thus the vector $(\tau_1,\ldots,\tau_m)$ is uniformly distributed on $[0,T]^m$ with constant probability density $T^{-m}$ so that
\begin{equation*}
\elecP(V_\delta^m(\zeta))
=\elecP(\{\hat\elec\in B\,:\,\hat\elec([0,T]\times[0,M])=m\})T^{-m}|S_\delta^m(\zeta)|
=0\,.
\end{equation*}
This ends the proof of the first part.

For the second part, assume $\lambda^n\not\to\lambda^\infty$ in $L^1((0,T))$, that is, there exists a subsequence (indexed again by $n$)
such that $S^n=\{(\tau,\lambda)\in[0,T]\times[0,\infty)\,:\,\lambda\in(\lambda^n(\tau),\lambda^\infty(\tau))\text{ or }\lambda\in(\lambda^\infty(\tau),\lambda^n(\tau))\}$ satisfies
$|S^n|\geq\delta$ for some $\delta>0$ and all $n$.
The probability that $\|\elec^n-\elec^\infty\|_{\meas([0,T])}\geq1$ is thus given by
\begin{equation*}
\elecP(\{\hat\elec\in B\,:\,\|\elec^n-\elec^\infty\|_{\meas([0,T])}\geq1\})
=\elecP(\{\hat\elec\in B\,:\,\hat\elec(S^n)\geq1\})
=1-e^{-|S^n|}
\geq1-e^{-\delta}\,.
\end{equation*}
If we introduce $B_n=\{\hat\elec\in B\,:\,\|\elec_j-\elec^\infty\|_{\meas([0,T])}\geq1\text{ for some }j\geq n\}$,
then this set is monotone, $B_n\supset B_{n+1}$, and has probability $\elecP(B_n)\geq1-e^{-\delta}$ due to the above.
Thus,
\begin{equation*}
\elecP\Big(\Big\{\hat\elec\in B\,:\,\limsup_{j\to\infty}\|\elec_j-\elec^\infty\|_{\meas([0,T])}\geq1\Big\}\Big)
=\elecP\left(\bigcap_{n=1}^\infty B_n\right)
\geq1-e^{-\delta}>0
\end{equation*}
due to the monotone convergence theorem.

Note that there is a subset of $B$ with positive probability such that we still have $\elec^n\to\elec^\infty$,
for instance all those $\hat\elec$ that do not contain any mass in $[0,T]\times[0,M]$.
\item We have
\begin{equation*}
\tfrac\alpha{\sigma^nN^n}\|\mathcal G_G^n\|_{\meas([0,T])}
=\tfrac1{N^n}\sum_{i=1}^{N^n}|\white(t_i^n)|\,.
\end{equation*}
By the strong law of large numbers, the right-hand side converges almost surely against the expected value of the $|\white(t_i^n)|$.
Since $\white(t_i^n)$ is normally distributed with mean $0$ and variance $1$, $|\white(t_i^n)|$ has expected value $\sqrt{2/\pi}$ so that almost surely
\begin{equation*}
\tfrac\alpha{\sigma^nN^n}\|\mathcal G_G^n\|_{\meas([0,T])}\to\sqrt{\tfrac2\pi}\,,
\end{equation*}
which direcly implies the first two statements.
Next assume $\lim_{n\to\infty}\tfrac{\sigma^nN^n}\alpha\in(0,\infty)$
so that we have almost sure boundedness of $\|\mathcal G_G^n\|_{\meas([0,T])}$.
Let $\varphi\in C([0,T])$.
For $\varepsilon>0$ let $\delta>0$ such that $|\varphi(t)-\varphi(\hat t)|<\varepsilon$ for all $|t-\hat t|<\delta$,
and let $0=\tau_0<\tau_1<\ldots<\tau_K=T$ with $\tau_{j+1}-\tau_j<\delta$ for all $j$.
For simplicity let us assume that for each $n$ we have $\{\tau_0,\ldots,\tau_K\}\subset\{t_0^1,\ldots,t_{N^n}^n\}$ (the argument can easily be adapted if this is not the case).
We have
\begin{multline*}
\left|\int_0^T\mathcal G_G^n\varphi\,\d t\right|
=\left|\sum_{j=1}^K\int_{\tau_{j-1}}^{\tau_j}\mathcal G_G^n\varphi\,\d t\right|
\leq\sum_{j=1}^K\left(|\varphi(\tau_j)|\left|\int_{\tau_{j-1}}^{\tau_j}\mathcal G_G^n\,\d t\right|+\varepsilon\int_{\tau_{j-1}}^{\tau_j}\left|\mathcal G_G^n\right|\,\d t\right)\\
=\varepsilon\|\mathcal G_G^n\|_{\meas([0,T])}+\sum_{j=1}^K|\varphi(\tau_j)|\left|\int_{\tau_{j-1}}^{\tau_j}\mathcal G_G^n\,\d t\right|
=\varepsilon\|\mathcal G_G^n\|_{\meas([0,T])}+\frac{\sigma^n}\alpha\sum_{j=1}^K|\varphi(\tau_j)|\left|\sum_{t_i^n\in(\tau_{j-1},\tau_j]}\white(t_i^n)\right|\\
\leq\varepsilon\|\mathcal G_G^n\|_{\meas([0,T])}+\|\varphi\|_{C([0,T])}\frac{\sigma^n}\alpha\sum_{j=1}^KK_j^n\left|\frac1{K_j^n}\sum_{t_i^n\in(\tau_{j-1},\tau_j]}\white(t_i^n)\right|\\
\leq\varepsilon\|\mathcal G_G^n\|_{\meas([0,T])}+\|\varphi\|_{C([0,T])}\frac{N^n\sigma^n}\alpha\sup_j\left|\frac1{K_j^n}\sum_{t_i^n\in(\tau_{j-1},\tau_j]}\white(t_i^n)\right|\,,
\end{multline*}
where $K_j^n$ denotes the number of summands in the interior sum.
Now by the law of large numbers, each of the absolute values on the right-hand side converges to zero almost surely
so that the desired result follows by the arbitrariness of $\varepsilon$.
\item By parts \ref{enm:elecCountConv} and \ref{enm:elecCountWeakConv} we almost surely have $\elec^n\to\elec^\infty$ (only for a subsequence in the case $\lambda^n\stackrel*\rightharpoonup\lambda^\infty$), which implies $\mathcal G_P^n\stackrel*\rightharpoonup\elec^\infty$.
Furthermore, part \ref{enm:backgNoiseConv} implies $\mathcal G_G^n\stackrel*\rightharpoonup0$ so that $\mathcal G^n=\mathcal G_P^n+\mathcal G_G^n\stackrel*\rightharpoonup\elec^\infty$.
\item For $t\in[t_{i-1}^n,t_i^n)$ we have
\begin{multline*}
|\tilde{\mathcal G}^n(t)|
=\frac{\left|\llbracket(t_i^n-t_{i-1}^n)\mathcal G^n(t)\rrbracket-(t_i^n-t_{i-1}^n)\mathcal G^n(t)\right|}{t_i^n-t_{i-1}^n}\\
\leq\left.\begin{cases}
\frac{\left|\llbracket(t_i^n-t_{i-1}^n)\mathcal G_G^n(t)\rrbracket-(t_i^n-t_{i-1}^n)\mathcal G_G^n(t)\right|}{t_i^n-t_{i-1}^n}&\text{if }\mathcal G_G^n(t)\geq0\\
\frac{\left|(t_i^n-t_{i-1}^n)\mathcal G_G^n(t)\right|}{t_i^n-t_{i-1}^n}&\text{else}
\end{cases}\right\}
\leq\left|\mathcal G_G^n(t)\right|\,,
\end{multline*}
thus $\|\tilde{\mathcal G}^n\|_{\meas([0,T])}\leq\|\mathcal G_G^n\|_{\meas([0,T])}\to0$ almost surely by part \ref{enm:backgNoiseConv}.\qedhere
\end{enumerate}
\end{proof}

Statements \ref{enm:elecCountConv} and \ref{enm:elecCountWeakConv} of the lemma show that as our scanning path and thus the temporally changing material density under the electron beam converge,
the events of electron detections will converge exactly to the events detected if a time-continuous scanning path were used
(at least up to a subsequence, if the material density under the beam only converges weakly).
Statement \ref{enm:backgNoiseConv} analyses the behavior of the background sensor noise and shows that measurements become useless if $\frac{\sigma^nN^n}\alpha\to\infty$,
since in that case the noisy background signal becomes infinitely large, swallowing up the electron count.
Thus, increasing the scan resolution and decreasing the dwell time can only be feasible if at the same time the variance of the background noise decreases sufficiently fast;
for given $\sigma$ one should choose a dwell time of at least $\Delta t\sim\sigma$.
Statement \ref{enm:niceData} then shows that the combined measured signal (the accumulated electron count and background noise at each measurement location)
indeed converges (in the weak sense) against the electron count during a time-continuous scan.
Finally, the last statement implies that for sufficiently decreasing background noise variance the electron counting signal becomes clean in the sense that it deviates only little from integer values.

For later reference we shall here also prove the following statement about points of a Poisson process.
\begin{lemma}\label{thm:mod1Dense}
Let $\hat N^n$ be a sequence with $\hat N^n\to\infty$ as $n\to\infty$ and $\elec=\sum_{j=1}^m\delta_{\tau_j}$ describe a Poisson point process on $[0,T]$ with positive intensity and $m$ points.
Then almost surely
\begin{equation*}
\{(\hat N^n\tfrac{\tau_1}T\!\!\mod1,\ldots,\hat N^n\tfrac{\tau_m}T\!\!\mod1)\,:\,n\in\N\}\text{ is dense on }[0,1]^m\,.
\end{equation*}
\end{lemma}
\begin{proof}
The times $\frac{\tau_1}T,\ldots,\frac{\tau_m}T$ are independently identically distributed on $[0,1]$. Define
\begin{align*}
B_\varepsilon(\boldsymbol\theta)&=(\theta_1-\varepsilon,\theta_1+\varepsilon)\times\ldots\times(\theta_m-\varepsilon,\theta_m+\varepsilon)\subset[0,1]^m\,,\\
S_\varepsilon^l(\boldsymbol\theta)&=\{\boldsymbol t\in[0,1]^m\,:\,(\hat N^nt_1\!\!\mod1,\ldots,\hat N^nt_m\!\!\mod1)\notin B_\varepsilon(\boldsymbol\theta)\,\forall n\geq l\}\,.
\end{align*}
Thus, we have to show that $S_\varepsilon^l(\boldsymbol\theta)$ is a nullset for all $\varepsilon>0$, $l\in\N$, and $\boldsymbol\theta\in(0,1)^m$.
Now $S_\varepsilon^l(\boldsymbol\theta)$ can be written as (compare Figure~\ref{fig:density})
\begin{align*}
S_\varepsilon^l(\boldsymbol\theta)=\bigcap_{k=l}^\infty T_\varepsilon^k(\boldsymbol\theta)
\quad\text{ for }
T_\varepsilon^k(\boldsymbol\theta)
&=\{\boldsymbol t\in[0,1]^m\,:\,(\hat N^kt_1\!\!\mod1,\ldots,\hat N^kt_m\!\!\mod1)\notin B_\varepsilon(\boldsymbol\theta)\}\\
&=[0,1]^m\setminus\prod_{j=1}^m\left[B_{\frac\varepsilon{\hat N^k}}\big(\tfrac{\theta_j}{\hat N^k}\big)\cup B_{\frac\varepsilon{\hat N^k}}\big(\tfrac{\theta_j}{\hat N^k}+\tfrac{1}{\hat N^k}\big)\cup\ldots\cup B_{\frac\varepsilon{\hat N^k}}\big(\tfrac{\theta_j}{\hat N^k}+\tfrac{\hat N^k-1}{\hat N^k}\big)\right]\,.
\end{align*}
Note that for every set $S\subset\R^m$ with finite perimeter there is some $k$ such that $|T_\varepsilon^k(\boldsymbol\theta)\cap S|/|S|<1-(2\varepsilon)^m/2$.
Indeed, for any $k\in\N$, the hypercube $[0,1]^m$ is tiled by $(\hat N^k)^m$ hypercubes of sidelength $\frac1{\hat N^k}$, each having a volume fraction $1-(2\varepsilon)^m$ inside $T_\varepsilon^k(\boldsymbol\theta)$ (see Figure~\ref{fig:density}).
By choosing $k$ sufficiently large, the little hypercubes also tile $S$ with an arbitrarily small error at the boundary $\partial S$ so that the volume fraction of $S$ inside $T_\varepsilon^k(\boldsymbol\theta)$ approaches $1-(2\varepsilon)^m$ as well.
Therefore, we may for $j\in\N$ recursively define
\begin{equation*}
S_1=T_\varepsilon^l(\boldsymbol\theta)\,,\quad
S_{j+1}=T_\varepsilon^{k_j}(\boldsymbol\theta)\cap S_j\,,\quad
\text{where }k_j\text{ satisfies }\tfrac{|T_\varepsilon^{k_{j}}(\boldsymbol\theta)\cap S_j|}{|S_j|}<1-\tfrac{(2\varepsilon)^m}2\,.
\end{equation*}
Then, $S_\varepsilon^l(\boldsymbol\theta)\subset\bigcap_{j=1}^\infty S_j$ so that (letting $|\cdot|$ denote Lebesgue measure)
\begin{equation*}
|S_\varepsilon^l(\boldsymbol\theta)|\leq\lim_{j\to\infty}|S_j|=\lim_{j\to\infty}|S_1|\frac{|S_2|}{|S_1|}\cdots\frac{|S_j|}{|S_{j-1}|}\leq\lim_{j\to\infty}(1-(2\varepsilon)^m/2)^j=0\,.\qedhere
\end{equation*}
\end{proof}

\begin{figure}
\centering
\setlength{\unitlength}{.35\linewidth}
\begin{picture}(1,1)
\put(0,0){\includegraphics[width=\unitlength]{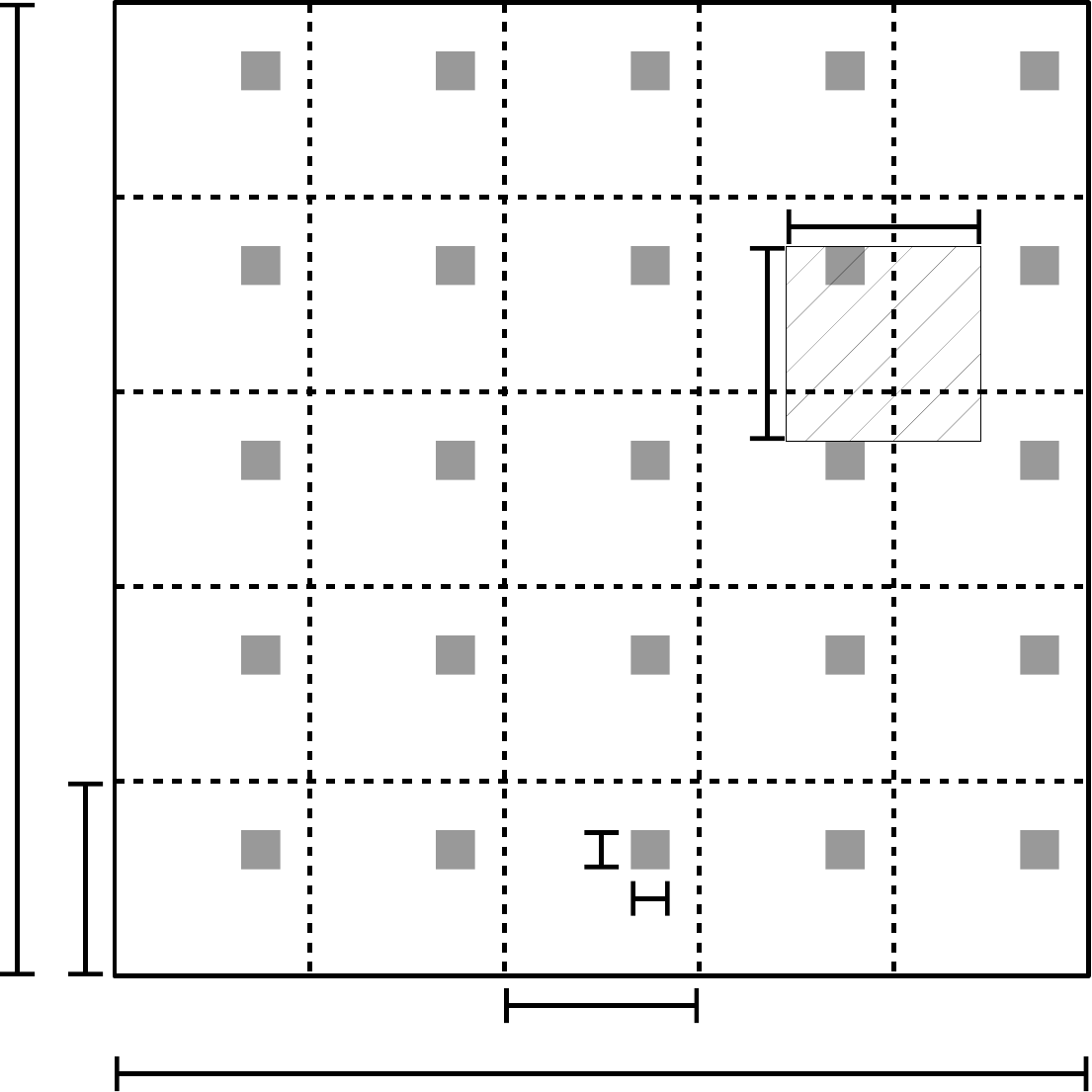}}
\put(.53,-.03){\small$1$}
\put(-.02,.53){\small$1$}
\put(.02,.2){\tiny$\tfrac1{\!\hat N^{\!k}\!}$}
\put(.49,.04){\tiny$1/\hat N^k$}
\put(.8,.68){\tiny$\bullet$}
\put(.82,.68){\small$\boldsymbol\theta$}
\put(.645,.68){\small$2\varepsilon$}
\put(.785,.8){\small$2\varepsilon$}
\put(.48,.22){\tiny$\tfrac{2\varepsilon}{\!\hat N^{\!k}\!}$}
\put(.565,.14){\tiny$\tfrac{2\!\varepsilon}{\!\hat N^{\!k}\!}$}
\put(1.13,.97){\vector(-2,-1){.17}}
\put(1.15,.95){$T_\varepsilon^k(\boldsymbol\theta)$}
\end{picture}
\caption{The set $T_\varepsilon^k(\boldsymbol\theta)$ from the proof of Lemma~\ref{thm:mod1Dense}, consisting of copies of $[0,1]^m\setminus B_\varepsilon(\boldsymbol\theta)$, scaled by $\frac1{\hat N^k}$.}
\label{fig:density}
\end{figure}
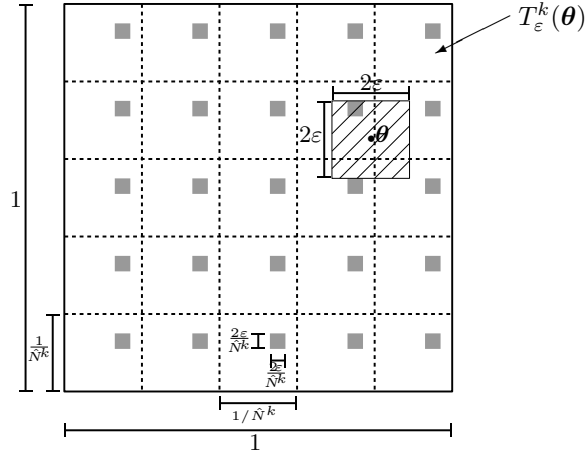

\subsection{Time-continuous scanning paths}
Here, we shall consider the case in which the piecewise constant scanning path $x^n(t)$ approximates a time-continuous scanning path as $n\to\infty$, that is,
\begin{equation*}
x^n\to x^\infty\text{ uniformly on }[0,T]\,,
\end{equation*}
where, for a fixed total acquisition time $T$, we assume the number of measurement locations $N^n$, the dwell time $\Delta t^n$, the signal acquisition times $t_i^n$, the path $x^n$, and the background noise standard deviation $\sigma^n$ to satisfy
\begin{equation*}
N^n\to\infty\,,\quad
\Delta t^n=\tfrac T{N^n}\,,\quad
t_i^n=i\Delta t^n\text{ for }i=1,\ldots,N^n\,,\quad
x^n(t)=x_i^n\text{ for }t\in[t_{i-1}^n,t_i^n)\,,\quad
\tfrac{\sigma^nN^n}\alpha\to0\,.
\end{equation*}
Consequently, assuming a smooth (for instance Lipschitz) true underlying material density $\den$, the material density $\lambda^n$ under the electron beam satisfies
\begin{equation*}
\lambda^n(t)=\den(x^n(t)+w(t))\to\den(x^\infty(t)+w(t))=\lambda^\infty(t)\text{ uniformly on }[0,T]\,.
\end{equation*}
Statement\,\ref{enm:niceData} of Lemma~\ref{thm:dataBehaviour} thus implies weak convergence of the measured signal,
\begin{equation*}
\mathcal G^n\stackrel*\rightharpoonup\elec^\infty
\text{ with }\mathcal G^n(t)=\frac{(\mathbf\img^n)_i-\mu}{\alpha(t_{i}^n-t_{i-1}^n)}\text{ for }t\in[t_{i-1}^n,t_i^n)\,,
\end{equation*}
where $\elec^\infty$ represents the signal belonging to $\lambda^\infty$.
In addition, by Statement\,\ref{enm:cleanData} of Lemma~\ref{thm:dataBehaviour} we have
\begin{equation*}
\frac1{\Delta t^n}\left\|\llbracket(\Delta t^n)\mathcal G^n\rrbracket-(\Delta t^n)\mathcal G^n\right\|_{\meas([0,T])}\to0\,.
\end{equation*}
In that case, the energy, whose minimizers yield estimates for the sample motion and material density, becomes
\begin{equation*}
E^n[\motion,\mathbf p]
=\sum_{i=1}^{N^n}\diss^n((\mathbf\img^n)_i,\Delta t^n\den[\mathbf p](x_i^n+w_i))+\frac1{2\diff}\sum_{i=1}^{N^n}\frac{|w_i-w_{i-1}|^2}{\Delta t^n}+\iota_A(\mathbf p)
\end{equation*}
(note that the data dissimilarity $\diss^n$ depends on $\sigma^n$ and thus has a superscript as well).
In this formulation, it is still inconvenient to analyze the energy convergence as $n\to\infty$, since the dimension of the argument $\motion$ changes with $n$.
Therefore, we reformulate the energy based on time-continuous representations of $\motion$ as
\begin{equation*}
\mathcal E^n[\mathcal W,\mathbf p]
=\begin{cases}
E^n[\motion,\mathbf p]&\text{if }\mathcal W:[0,T]\to\R^2\text{ is the piecewise affine interpolation}\\&\quad\text{ of the values }\motion=(w_0,\ldots,w_{N^n})\text{ at positions }(t_0^n,\ldots,t_{N^n}^n)\\
\infty&\text{else,}
\end{cases}
\end{equation*}
where consistently with our previous notation, we use calligraphic capitals for functions derived from discrete vectors.
We shall show that this energy $\Gamma$-converges (up to constants) against the energy
\begin{equation*}
\mathcal E^\infty[\mathcal W,\mathbf p]
=\begin{cases}
\diss_{\mathrm{KL}}(\elec^\infty,\den[\mathbf p](x^\infty(\cdot)+\mathcal W(\cdot)))
+\frac1{2\diff}|\mathcal W|_{H^1((0,T))^2}^2
+\iota_A(\mathbf p)\,,
&\text{if }\mathcal W(0)=w_0\\
\infty&\text{else,}
\end{cases}
\end{equation*}
where $|\cdot|_{H^1((0,T))^2}$ is the $H^1$-seminorm and $\diss_{\mathrm{KL}}$ is the Kullback--Leibler divergence
\begin{equation*}
\diss_{\mathrm{KL}}(\elec,\den)=\int_0^T\den(t)\,\d t-\int_0^T\log\den(t)\,\d\elec(t)\,.
\end{equation*}
Thus, in the limit, the motion estimate $\mathcal W$ is regularized in $H^1$, and the material density is compared to the electron detections via the Kullback--Leibler divergence between measures.

\begin{theorem}[Limit reconstruction for time-continuous scanning paths]\label{thm:GammaLimit}
Let $\mathbf p\mapsto\den[\mathbf p]$ be continuous from $A$ to $C^{0,1}(\Omega)$ with $\den[\mathbf p]>0$ bounded away from zero for any $\mathbf p$ (this holds true for our particular choice of parameterization).
There exists a sequence of constants $C^n=C^n(\mathcal G^n,\sigma^n)$ such that
with respect to weak convergence in $H^1((0,T))^2$ and (strong) convergence in $\R^{\numFitParams}$ we have
\begin{equation*}
\Gamma-\lim_{n\to\infty}
\mathcal E^n-C^n
=\mathcal E^\infty\,.
\end{equation*}
Furthermore, minimizers of $\mathcal E^n$ converge in the same topology against minimizers of $\mathcal E^\infty$.
\end{theorem}
\begin{proof}
\emph{$\liminf$-inequality}:
Let $\mathcal W^n\rightharpoonup\mathcal W$ in $H^1((0,T))^2$ and $\mathbf p^n\to\mathbf p$.
Without loss of generality we may assume $\mathcal E^n[\mathcal W^n,\mathbf p^n]-C^n<C<\infty$ for all $n$ (else we may either restrict to a subsequence or there is nothing to show).
Due to the compact embedding of $H^1((0,T))$ in $C^0([0,T])$ combined with $\mathcal W^n(0)=w_0$ for all $n$, we get $\mathcal W(0)=w_0$.

Furthermore, note that
\begin{equation*}
\mathcal E^n[\mathcal W^n,\mathbf p^n]=
\int_0^T\frac{\diss^n(\Delta t^n\alpha\mathcal G^n(t)+\mu,\Delta t^n\den[\mathbf p^n](x^n(t)+\mathcal W^n(\Delta t^n\lceil\tfrac{t}{\Delta t^n}\rceil)))}{\Delta t^n}\,\d t
+\frac1{2\diff}|\mathcal W^n|_{H^1((0,T))^2}^2
+\iota_A(\mathbf p^n)\,.
\end{equation*}
Since $|\cdot|_{H^1((0,T))^2}^2$ and $\iota_A$ are sequentially lower semi-continuous under weak convergence in $H^1((0,w))^2$ and convergence in $\R^{\numFitParams}$, respectively,
it suffices to show that the integral minus
\begin{equation*}
C^n=\frac1{\Delta t^n}\int_0^T\tilde C(\alpha\Delta t^n\mathcal G^n(t),\alpha,\sigma^n)-\llbracket\Delta t^n\mathcal G^n(t)\rrbracket\log\Delta t^n\,\d t
\end{equation*}
(with $\tilde C$ from Lemma~\ref{thm:dataTerm})
converges up to a subsequence against $\diss_{\mathrm{KL}}(\elec^\infty,\den[\mathbf p](x^\infty(\cdot)+\mathcal W(\cdot)))$.

We have $\den[\mathbf p^n]\to\den[\mathbf p]$ in $C^{0,1}(\Omega)$ as well as $x^n\to x^\infty$ in $L^\infty((0,T))$ and $\mathcal W^n\to\mathcal W$ in $C^0((0,T))$ (upon extracting a subsequence) so that also
\begin{equation*}
u^n(t)
=\den[\mathbf p^n](x^n(t)+\mathcal W^n(\Delta t^n\lceil\tfrac{t}{\Delta t^n}\rceil))
\to
\den[\mathbf p](x^\infty(t)+\mathcal W(t))
=u^\infty(t)
\end{equation*}
uniformly.
Furthermore, for $n$ large enough Lemma~\ref{thm:dataTerm} implies
\begin{multline*}
\frac{\diss^n(\Delta t^n\alpha\mathcal G^n(t)+\mu,\Delta t^nu^n(t))}{\Delta t^n}
-\frac{\tilde C(\alpha\Delta t^n\mathcal G^n(t),\alpha,\sigma^n)}{\Delta t^n}
+\frac{\llbracket\Delta t^n\mathcal G^n(t)\rrbracket}{\Delta t^n}\log\Delta t^n\\
=u^n(t)-\frac{\llbracket\Delta t^n\mathcal G^n(t)\rrbracket}{\Delta t^n}\log u^n(t)+O(\tfrac{\sigma^n}{\alpha\Delta t^n})
\end{multline*}
since $\underline z\leq\den[\mathbf p^n]\leq\overline z$ for some $0<\underline z<\overline z$, $\Delta t^n\mathcal G^n(t)\leq\|\mathcal G^n\|_{\meas([0,T])}$ is uniformly bounded,
and $|\llbracket\Delta t^n\mathcal G^n(t)\rrbracket-\Delta t^n\mathcal G^n(t)|
\leq\frac{\|\llbracket\Delta t^n\mathcal G^n\rrbracket-\Delta t^n\mathcal G^n\|_{\meas([0,T])}}{\Delta t^n}\to0$ uniformly in $t$.
Thus, due to the uniform convergence of $u^n$ and the weak-* convergence of $\mathcal G^n$
(which implies $\frac{\llbracket\Delta t^n\mathcal G^n\rrbracket}{\Delta t^n}\stackrel*\rightharpoonup\elec^\infty$) we obtain
\begin{multline*}
\int_0^T\frac{\diss^n(\Delta t^n\alpha\mathcal G^n(t)+\mu,\Delta t^nu^n(t))}{\Delta t^n}\,\d t-C^n
=\int_0^Tu^n(t)-\frac{\llbracket\Delta t^n\mathcal G^n(t)\rrbracket}{\Delta t^n}\log u^n(t)\,\d t+O(\tfrac{T\sigma^n}{\alpha\Delta t^n})\\
\to\int_0^Tu^\infty(t)\,\d t-\int_0^T\log u^\infty(t)\,\d\elec^\infty(t)
=\diss_{\mathrm{KL}}(\elec^\infty,u^\infty)\,,
\end{multline*}
as desired.
Indeed, for any $g^n\stackrel*\rightharpoonup g$ we have
\begin{multline*}
\left|\int_0^T\log u^n\,\d g^n-\int_0^T\log u^\infty\,\d g\right|
\leq\left|\int_0^T\log u^n-\log u^\infty\,\d g^n\right|+\left|\int_0^T\log u^\infty\,\d(g-g^n)\right|\\
\leq\|\log u^n-\log u^\infty\|_{L^\infty((0,T))}\|g^n\|_{\meas([0,T])}+\left|\int_0^T\log u^\infty\,\d(g-g^n)\right|
\to0\,.
\end{multline*}

\emph{$\limsup$-inequality}:
Let $\mathcal W\in H^1((0,T))^2$ and $\mathbf p\in\R^{\numFitParams}$ be given with $\mathcal E^\infty[\mathcal W,\mathbf p]<\infty$.
As recovery sequence we choose $\mathbf p^n=\mathbf p$ and $\mathcal W^n$ the piecewise affine interpolation of the points $(t_i^n,\mathcal W(t_i^n))$, $i=0,\ldots,N^n$,
which even converges strongly in $H^1((0,T))^2$.
With this choice, we obtain
\begin{equation*}
\mathcal E^n[\mathcal W^n,\mathbf p^n]=
\int_0^T\frac{\diss^n(\Delta t^n\alpha\mathcal G^n(t)+\mu,\Delta t^n\den[\mathbf p](x^n(t)+\mathcal W^n(\Delta t^n\lceil\tfrac{t}{\Delta t^n}\rceil)))}{\Delta t^n}\,\d t
+\frac1{2\diff}|\mathcal W^n|_{H^1((0,T))^2}^2
+\iota_A(\mathbf p)\,.
\end{equation*}
The latter two terms converge against $\frac1{2\diff}|\mathcal W|_{H^1((0,T))^2}^2+\iota_A(\mathbf p)$, while the integral minus $C^n$ converges to $\diss_{\mathrm{KL}}(\elec^\infty,\den[\mathbf p](x^\infty(\cdot)+\mathcal W(\cdot)))$ as in the proof of the $\liminf$-inequality
so that $\mathcal E^n[\mathcal W^n,\mathbf p^n]-C^n\to\mathcal E^\infty[\mathcal W,\mathbf p]$ as desired.

\emph{Convergence of minimizers}:
For any $\hat{\mathbf p}\in A$, we have $\min_{\mathcal W,\mathbf p}\mathcal E^n[\mathcal W,\mathbf p]\leq\mathcal E^n[0,\hat{\mathbf p}]$, and the right-hand side is uniformly bounded in $n$.
Since $\frac1{2\diff}|\mathcal W|_{H^1((0,T))^2}^2+\iota_A(\mathbf p)$ forms part of each $\mathcal E^n$,
this implies (using $\mathcal W^n(0)=0$ and the Poincar\'e inequality) that the set of minimizers of the $\mathcal E^n$ is uniformly bounded (and thus sequentially compact with respect to our chosen topology) in $H^1((0,T))^2\times\R^{\numFitParams}$,
which in turn is well-known to result in any sequence of minimizers having a subsequence converging to a minimizer of $\mathcal E^\infty$.
\end{proof}

That the limit energy only stays finite for motions $\mathcal W$ of $H^1$-regularity may seem a little counterintuitive
since Brownian motion almost surely has no weak derivative, however, in the MAP estimate, we retrieve very special realizations which may indeed have additional regularity.

\begin{remark}[Full coverage via Peano curve]
In some contexts, for instance for biological samples that can only be scanned at much lower magnification,
it might be advantageous to fully cover the whole sample $\Omega$ during the image acquisition
(that is, to traverse every point in $\Omega$ at least once along the scan path),
since otherwise small objects might be overlooked in between the scan locations (for instance between the rows for row-wise scans).
This can be achieved using a space-filling curve such as a Peano curve $x^\infty$ as scan path.

To this end, consider a standard sequence of piecewise affine curves $y^n:[0,T]\to\Omega$ with constant absolute velocity
such that $y^n$ converges uniformly to the Peano curve $x^\infty:[0,T]\to\Omega$.
Let $N^n$ denote the number of corners of $y^n$ and take these as measurement locations $x_1^n,\ldots,x_{N^n}^n$ with dwell time $\Delta t^n=\frac T{N^n}$.
Then the corresponding beam scanning path $x^n$ also converges uniformly against $x^\infty$ so that our previous $\Gamma$-convergence result applies.
In particular, in the limit $n\to\infty$ the data term $\diss_{\mathrm{KL}}(\elec^\infty,\den[\mathbf p](x^\infty(\cdot)+\mathcal W(\cdot)))$
compares the measurement to the estimated material density at every point in $\Omega$ (up to the Brownian sample motion).

Note that since each point in $\Omega$ is covered one might be tempted to rewrite the data term and the $H^1$-regularization of $\mathcal W$ as a space integral over $\Omega$,
however, this is not possible since a space-filling curve $x^\infty$ can never be injective.
Thus, a sampling of the entire domain can only be expressed as a time-like variational problem, but not a space-like one.
\end{remark}

\begin{remark}[Constant number of continuous rows]
Typically, the electron beam scans the sample row-wise as described in Section\,\ref{sec:scanMode}.
We may consider the case in which the number $N_2$ of scanned rows as well as the scan time $T_{\text{row}}$ per row and the waiting time $\Delta T$ between two rows stay the same,
but the number $N_1^n=\frac{T_{\text{row}}}{\Delta t^n}=\frac a{\Delta x_1^n}$ of acquired pixels per row converges to infinity.
In this case, our $\Gamma$-convergence result applies to each single row scan.

Setting $T=T_{\text{row}}$ we may split each function $f(t)$ of time up into the time intervals corresponding to the different scan rows according to
\begin{equation*}
f_j:[0,T]\to\R^2\,,\quad
f_j(t)=f(t+(j-1)(T_{\text{row}}+\Delta T))\,,\quad
j=1,\ldots,N_2\,.
\end{equation*}
In particular, we apply this notation to the motion estimate $f=\mathcal W$, the scan path $f=x^n$, and the signal $f=\mathcal G^n$.
Then, our model energy can be rewritten as
\begin{equation*}
E^n[\motion,\mathbf p]
=\mathcal E^n[\mathcal W,\mathbf p]
=\sum_{j=1}^{N_2}\mathcal E^n_j[\mathcal W_j,\mathbf p]+\frac1{2\diff}\frac{|\mathcal W_j(0)-\mathcal W_{j-1}(T)|^2}{\Delta T}\,,
\end{equation*}
where we simply set $\mathcal W_{0}=0$ and where $\mathcal E^n_j$ represents the energy for scan path $x^n_j$ and signal $\mathcal G^n_j$ (note that for the $j$\textsuperscript{th} energy $\mathcal E^n_j$ the initial accumulated motion is $w_0=\mathcal W_j(0)$).
Since the $\mathcal E^n_j$ $\Gamma$-converge against the $\mathcal E^\infty_j$
and $\sum_{j=1}^{N_2}\frac1{2\diff}\frac{|\mathcal W_j(0)-\mathcal W_{j-1}(T)|^2}{\Delta T}$ represents a continuous perturbation, the $\Gamma$-limit of the full energy is
\begin{equation*}
\mathcal E^\infty[\mathcal W,\mathbf p]
=\sum_{j=1}^{N_2}\mathcal E^\infty_j[\mathcal W_j,\mathbf p]+\frac1{2\diff}\frac{|\mathcal W_j(0)-\mathcal W_{j-1}(T)|^2}{\Delta T}\,.
\end{equation*}
Thus, the motion estimate is $H^1$-regularized separately for each row, and the difference between the last and the first shift of each row are penalized quadratically in addition.
\end{remark}

\subsection{Finer row and column resolution}
We now consider the case in which there are not only more and more pixels per row, but in which the number of rows $N_2^n$ also increases to infinity.
Of course, to perform such measurements in finite time, not only the dwell time $\Delta t^n$ has to decrease to zero, but also the waiting time $\Delta T^n$ between consecutive rows.
At first sight, this may seem unrealistic since a sufficient waiting time is necessary after any large motion,
however, if instead of scanning each row from left to right one scans the rows in alternating directions, obtaining a snake-like scanning path,
then the motion between consecutive rows actually is of the same order as the motion between two pixels so that no waiting time is necessary.
For simplicity we shall thus set the waiting time $\Delta T$ between consecutive rows to zero;
its inclusion would not lead to a qualitatively different analysis.
Thus, we choose
\begin{gather*}
N_1^n,N_2^n\to\infty\,,\quad
N^n=N_1^nN_2^n\,,\quad
\Delta t^n=\tfrac T{N^n}\,,\quad
t_i^n=i\Delta t^n\text{ for }i=1,\ldots,N^n\,,\quad
\tfrac{\sigma^nN^n}\alpha\to0\,,\\
\Delta x_1^n=\tfrac a{N_1^n}\,,\quad
\Delta x_2^n=\tfrac a{N_2^n}\,,\quad
x^n(t)=x_{ml}^n=(m\Delta x_1^n,l\Delta x_2^n)\text{ for }t\in[t_{i-1}^n,t_i^n)\text{ and }m,l\text{ according to \eref{eq:STEMIndexing}}\,.\quad
\end{gather*}
We will see that---for our chosen stochastic coupling---the variational model does not converge in this case.

Again we have to identify how the measured signal behaves as $n\to\infty$.
We first note that the material density $\lambda^n(t)=\den(x^n(t)+w(t))$ under the electron beam converges weakly against some row-wise average density.
\begin{lemma}
Denoting by $v=\frac{a}{T}$ the average vertical speed of the electron beam, we have
\begin{equation*}
\lambda^n(t)=\den(x^n(t)+w(t))
\stackrel*\rightharpoonup
\lambda^\infty(t)
=\frac1a\int_0^a\den((s,vt)+w(t))\,\d s
\quad\text{ in }L^\infty((0,T))\,.
\end{equation*}
\end{lemma}
\begin{proof}
It suffices to show $\int_\alpha^\beta\lambda^n(t)\,\d t\to\int_\alpha^\beta\lambda^\infty(t)\,\d t=\frac1{av}\int_{[v\alpha,v\beta]\times[0,a]}\den(x+w(x_2/v))\,\d x$ as $n\to\infty$ for arbitrary $\alpha,\beta\in[0,T]$,
since linear combinations of characteristic functions are dense in $L^1((0,T))$.
Now let $L_u$ be the Lipschitz constant of $\den$
and let $\varepsilon:[0,\infty]\to[0,\infty)$ with $\lim_{\delta\to0}\varepsilon(\delta)=0$ such that $|w(t)-w(\hat t)|\leq\varepsilon(\delta)$ for all $|t-\hat t|\leq\delta$.
Then we have
\begin{align*}
&\int_\alpha^\beta\den(x^n(t)+w(t))\,\d t\\
&=\Delta t^n\sum_{l=\lfloor\frac{v\alpha}{\Delta x_2^n}\rfloor}^{\lfloor\frac{v\beta}{\Delta x_2^n}\rfloor}\sum_{m=1}^{N_1^n}\den(x_{ml}^n+w((lN_1^n+m)\Delta t^n))+O(N_1^n\Delta t^n)\\
&=\frac{\Delta t^n}{\Delta x_1^n}\sum_{l=\lfloor\frac{v\alpha}{\Delta x_2^n}\rfloor}^{\lfloor\frac{v\beta}{\Delta x_2^n}\rfloor}\int_0^a\den\big((s,l\Delta x_2^n)+w\big(\tfrac{l\Delta x_2^n}v\big)\big)\,\d s
+O(N_1^n\Delta t^n+L_u(\Delta x_1^n+\varepsilon(N_1^n\Delta t^n)))\\
&=\frac{\Delta t^n}{\Delta x_1^n\Delta x_2^n}\int_{\lfloor\frac{v\alpha}{\Delta x_2^n}\rfloor\Delta x_2^n}^{\lfloor\frac{v\beta}{\Delta x_2^n}\rfloor\Delta x_2^n}\int_0^a\den(x+w(x_2/v))\,\d x
+O\big(N_1^n\Delta t^n+L_u\big(\Delta x_1^n+\Delta x_2^n+\varepsilon\big(N_1^n\Delta t^n+\tfrac{\Delta x_2^n}v\big)\big)\big)\,,
\end{align*}
which converges against the desired limit.
\end{proof}
Thus, by Statement\,\ref{enm:elecCountWeakConv} of Lemma~\ref{thm:dataBehaviour} a subsequence of the electron detection events (still indexed by $n$) converges, $\elec^n\to\elec^\infty$,
where $\elec^\infty$ represents the signal belonging to $\lambda^\infty$.
Since all $\elec^n$ and $\elec^\infty$ are sums of Dirac measures, this actually implies $\elec^n=\elec^\infty$ for $n$ large enough.
Therefore, restricting to the subsequence and large enough $n$, the measured signal satisfies
\begin{equation*}
\mathbf\img^n=\alpha\mathbf\img_P^n+\mathbf\img_G^n\,,
\quad\text{ for }
\mathbf\img_P^n=(\elec^\infty([t_0^n,t_1^n)),\ldots,\elec^\infty([t_{N^n-1}^n,t_{N^n}^n)))\,,\quad
\mathbf\img_G^n=\mu+\sigma^n(\white(t_1^n),\ldots,\white(t_{N^n}^n))\,,
\end{equation*}
where, due to Statements\,\ref{enm:backgNoiseConv} and \ref{enm:cleanData} of Lemma~\ref{thm:dataBehaviour}, we have
\begin{equation*}
\sum_{i=1}^{N^n}\left|\tfrac{(\mathbf\img_G^n)_i-\mu}\alpha\right|
=\|\mathcal G_G^n\|_{\meas([0,T])}\to0\,,
\qquad
\sum_{i=1}^{N^n}\left|\left\llbracket\tfrac{(\mathbf\img^n)_i-\mu}\alpha\right\rrbracket-\tfrac{(\mathbf\img^n)_i-\mu}\alpha\right|
=\|\tilde{\mathcal G}^n\|_{\meas([0,T])}\to0\,.
\end{equation*}

To compare scan paths of different resolution we shall again extend the vector of estimated Brownian motions to a continuous function, which this time shall be defined on $\Omega$.
Thus, we set
\begin{equation*}
\mathcal E^n[\mathcal W,\mathbf p]
=\begin{cases}
E^n[\motion,\mathbf p]&\text{if }\mathcal W:\Omega\to\R^2\text{ is the piecewise bilinear interpolation of the}\\&\quad\text{ values }\motion=(w_1,\ldots,w_{N^n})\text{ at positions }(x_{11}^n,x_{21}^n,\ldots,x_{N_1^n1}^n,x_{12}^n,\ldots,x_{N_1^nN_2^n}^n)\\
\infty&\text{else.}
\end{cases}
\end{equation*}
As the following result shows, this variational model does not have a limit model in general.

Note that, for the sake of simplicity, we dropped the condition $\mathcal W(0)=w_0$, when defining $\mathcal W$ on $\Omega$ instead of $[0,T]$. This leads to some invariance that causes $\mathcal E^n$ to have infinitely many minimizers (indeed, a constant shift in $\mathcal W$ can be compensated for by the applying the same shift to the atom positions), but is not the reason why there is no limit model.
\begin{theorem}[Nonexistence of limit model]\label{thm:noGammaLimit}
Again, let $\mathbf p\mapsto\den[\mathbf p]$ be continuous from $A$ to $C^{0,1}(\Omega)$ with $\den[\mathbf p]>0$ bounded away from zero for any $\mathbf p$, and let $C^n$ be the sequence of constants from Theorem~\ref{thm:GammaLimit}.
Almost surely, with respect to weak convergence in $H^1(\Omega)^2$ and (strong) convergence in $\R^{\numFitParams}$, we have
\begin{align*}
\Gamma-\liminf_{n\to\infty}\mathcal E^n[\mathcal{W},\mathbf p]-C^n
=&\iota_A(\mathbf p)+\frac a{2\diff T}\int_0^a|\partial_{x_2}\mathcal W(a,x_2)|^2\,\d x_2+\iota_{\partial_{x_1}\mathcal W=0}(\mathcal W)\\
&+\frac T{|\Omega|}\int_\Omega\den[\mathbf p](x+\mathcal W(x))\,\d x-\int_0^a\log\max_{x\in[0,a]\times\{s\}}\den[\mathbf p](x+\mathcal W(x))\,\d\elec^\infty(\tfrac Tas)\,,\\
\Gamma-\limsup_{n\to\infty}\mathcal E^n[\mathcal{W},\mathbf p]-C^n
=&\iota_A(\mathbf p)+\frac a{2\diff T}\int_0^a|\partial_{x_2}\mathcal W(a,x_2)|^2\,\d x_2+\iota_{\partial_{x_1}\mathcal W=0}(\mathcal W)\\
&+\frac T{|\Omega|}\int_\Omega\den[\mathbf p](x+\mathcal W(x))\,\d x-\int_0^a\log\min_{x\in[0,a]\times\{s\}}\den[\mathbf p](x+\mathcal W(x))\,\d\elec^\infty(\tfrac Tas)\,,
\end{align*}
where $\iota_{\partial_{x_1}\mathcal W=0}(\mathcal W)=0$ if $\mathcal W(\cdot,x_2)$ is constant for almost all $x_2\in[0,a]$ and $\iota_{\partial_{x_1}\mathcal W=0}(\mathcal W)=\infty$ else
and where $\elec^\infty(\tfrac Ta\cdot)$ is to be interpreted as the pushforward (image measure) of $\elec^\infty$ under $s\mapsto\tfrac aTs$.
\end{theorem}
\begin{proof}
Without loss of generality let $\elec^\infty=\hat\elec_{\lambda^\infty}=\sum_{j=1}^m\delta_{\tau_j}$.
Furthermore, we only consider $n$ large enough, such that each measurement interval $[t_{i-1}^n,t_i^n)$ only contains a single electron detection $\tau_j$. Thus, the $j$\textsuperscript{th} electron arrives at time $\tau_j\in[0,T]$ and is detected
at the $c_j^n=\left\lceil\frac{\tau_j\,\mathrm{mod}\,(N_1^n\Delta t^n)}{\Delta t^n}\right\rceil$\textsuperscript{th} pixel
in the $r_j^n=\left\lceil\frac{\tau_j}{N_1^n\Delta t^n}\right\rceil$\textsuperscript{th} row,
\begin{equation*}
(\mathbf\img^n)_i=\begin{cases}(\mathbf\img_G^n)_i+\alpha&\text{if there exists a $j$ with }i=i_j^n=r_j^nN_1^n+c_j^n\\(\mathbf\img_G^n)_i&\text{else.}\end{cases}
\end{equation*}

Now let $\mathcal W^n\rightharpoonup\mathcal W$ in $H^1(\Omega)^2$ and $\mathbf p^n\to\mathbf p$.
Using Lemma~\ref{thm:dataTerm} and the constant $C^n$ from Theorem~\ref{thm:GammaLimit} depending on the data, $\alpha$, and $\sigma^n$, the data term is given by
\begin{align*}
D^n_\data
&=\sum_{i=1}^{N^n}\diss^n((\mathbf\img^n)_{i},\Delta t^n\den[\mathbf p^n](x^n_i+\mathcal W^n(x^n_i))\\
&=C^n+O(\tfrac{N^n\sigma^n}{\alpha})
-\sum_{i=1}^{N^n}\left\llbracket\frac{(\mathbf\img^n)_{i}-\mu}\alpha\right\rrbracket\log\den[\mathbf p^n](x^n_{i}+\mathcal W^n(x^n_{i}))
+\sum_{i=1}^{N^n}\Delta t^n\den[\mathbf p^n](x^n_{i}+\mathcal W^n(x^n_{i}))\,.
\end{align*}
The limit as $n\to\infty$ of the last sum is obtained as
\begin{equation*}
\sum_{i=1}^{N^n}\Delta t^n\den[\mathbf p^n](x^n_{i}+\mathcal W^n(x^n_{i}))
=\frac T{a^2}\sum_{c=1}^{N_1^n}\sum_{r=1}^{N_2^n}\Delta x_1^n\Delta x_2^n\den[\mathbf p^n](x^n_{cr}+\mathcal W^n(x^n_{cr}))
\to\frac T{|\Omega|}\int_\Omega\den[\mathbf p](x+\mathcal W(x))\,\d x
\end{equation*}
due to the uniform convergence of $\den[\mathbf p^n]$ and the convergence $\mathcal W^n\to\mathcal W$ in $L^1(\Omega)$ because of the compact embedding $H^1(\Omega)\hookrightarrow L^1(\Omega)$,
while the second sum can be written as
\begin{align*}
\sum_{i=1}^{N^n}\left\llbracket\frac{(\mathbf\img^n)_{i}-\mu}\alpha\right\rrbracket\log\den[\mathbf p^n](x^n_{i}+\mathcal W^n(x^n_{i}))
=O(\|\tilde{\mathcal G}^n\|_{\meas([0,T])})
+\sum_{i=1}^{N^n}\frac{(\mathbf\img^n)_{i}-\mu}\alpha\log\den[\mathbf p^n](x^n_{i}+\mathcal W^n(x^n_{i}))\\
=O(\|\tilde{\mathcal G}^n\|_{\meas([0,T])})
+\sum_{j=1}^{m}\log\den[\mathbf p^n](x^n_{i_j^n}+\mathcal W^n(x^n_{i_j^n}))
+\sum_{i=1}^{N^n}\frac{(\mathbf\img_G^n)_{i}-\mu}\alpha\log\den[\mathbf p^n](x^n_{i}+\mathcal W^n(x^n_{i}))\\
=O(\|\tilde{\mathcal G}^n\|_{\meas([0,T])}+\|\mathcal G_G^n\|_{\meas([0,T])})
+\sum_{j=1}^{m}\log\den[\mathbf p^n]((c_j^n\Delta x_1^n,r_j^n\Delta x_2^n)+\mathcal W^n(c_j^n\Delta x_1^n,r_j^n\Delta x_2^n))\,.
\end{align*}
Now almost surely (with respect to the distribution of $\hat\elec$ or equivalently $(\tau_1,\ldots,\tau_m)$)
\begin{align*}
\Delta x_2^n(r_1^n,\ldots,r_m^n)&\to\tfrac aT(\tau_1,\ldots,\tau_m)\text{ as }n\to\infty\,,\text{ while}\\
\Delta x_1^n(c_1^n,\ldots,c_m^n)&=\tfrac a{N_1^n}(\lceil N_1^n(N_2^n\tfrac{\tau_1}T \mod 1)\rceil,\ldots,\lceil N_1^n(N_2^n\tfrac{\tau_m}T \mod 1)\rceil)\,,
\,n\in\N\,,\text{ is dense on }[0,a]^m\,.
\end{align*}
The latter statement follows from Lemma~\ref{thm:mod1Dense} noting $\tfrac a{N_1^n}\lceil N_1^n(N_2^n\tfrac{\tau_j}T \mod 1)\rceil=a(N_2^n\tfrac{\tau_j}T \mod 1)+O(\frac a{N_1^n})$.
Therefore,
\begin{equation*}
\left\{\begin{array}{l}\displaystyle\limsup_{n\to\infty}\\\displaystyle\liminf_{n\to\infty}\end{array}\right\}
\sum_{j=1}^{m}\log\den[\mathbf p^n]\left({c_j^n\Delta x_1^n\choose r_j^n\Delta x_2^n}+\mathcal W^n{c_j^n\Delta x_1^n\choose r_j^n\Delta x_2^n}\right)
=\sum_{j=1}^{m}\left\{\begin{array}{l}\displaystyle\sup_{x_1\in[0,a]}\\\displaystyle\inf_{x_1\in[0,a]}\end{array}\right\}\log\den[\mathbf p]\left({x_1\choose\frac{a\tau_j}T}+\mathcal W{x_1\choose\frac{a\tau_j}T}\right)
\end{equation*}
so that with the notation $\elec^\infty(\frac Tas)=\sum_{j=1}^m\delta_{a\tau_j/T}(s)$ we obtain
\begin{align*}
\liminf_{n\to\infty}D^n_\data-C^n&=\frac T{|\Omega|}\int_\Omega\den[\mathbf p](x+\mathcal W(x))\,\d x-\int_0^a\log\max_{x\in[0,a]\times\{s\}}\den[\mathbf p](x+\mathcal W(x))\,\d\elec^\infty(\tfrac Tas)\,,\\
\limsup_{n\to\infty}D^n_\data-C^n&=\frac T{|\Omega|}\int_\Omega\den[\mathbf p](x+\mathcal W(x))\,\d x-\int_0^a\log\min_{x\in[0,a]\times\{s\}}\den[\mathbf p](x+\mathcal W(x))\,\d\elec^\infty(\tfrac Tas)\,.
\end{align*}

Setting $\mathcal W^n(x^n_{N_1^n0}):=\mathcal W^n(x^n_{N_1^n1})$, the remaining terms of $\mathcal E^n[\mathcal W^n,\mathbf p^n]$ can be expressed as
\begin{align*}
\tilde{\mathcal E}^n
&=\iota_A(\mathbf p^n)
+\frac{1}{2\diff\Delta t^n}\sum_{r=1}^{N_2^n}\Bigg[|\mathcal W^n(x^n_{1r})-\mathcal W^n(x^n_{N_1^n(r-1)})|^2
+\sum_{c=2}^{N_1^n}|\mathcal W^n(x^n_{cr})-\mathcal W^n(x^n_{(c-1)r})|^2\Bigg]\\
&=\iota_A(\mathbf p^n)
+\frac{1}{2\diff\Delta t^n}\sum_{r=1}^{N_2^n}\Bigg[|\mathcal W^n(x^n_{1r})-\mathcal W^n(x^n_{N_1^n(r-1)})|^2
+\Delta x_1^n\int_{\Delta x_1^n}^a|\partial_{x_1}\mathcal W^n(x_1,r\Delta x_2^n)|^2\,\d x_2\Bigg]
\end{align*}
so that $\liminf_{n\to\infty}\tilde{\mathcal E}^n=\infty$ unless
\begin{equation*}
\sup_{x_2\in[\Delta x_2^n,a]}\int_{\Delta x_1^n}^a|\partial_{x_1}\mathcal W^n(x_1,x_2)|^2\,\d x_2
=\sup_{r=1,\ldots,N_2^n}\int_{\Delta x_1^n}^a|\partial_{x_1}\mathcal W^n(x_1,r\Delta x_2^n)|^2\,\d x_2
\to0\,,
\end{equation*}
which due to $\mathcal W^n\rightharpoonup\mathcal W$ implies $\partial_{x_1}\mathcal W=0$ almost everywhere. Furthermore,
\begin{align*}
\tilde{\mathcal E}^n
&=\iota_A(\mathbf p^n)
+\frac{N_1^n}{2\diff\Delta t^n}\sum_{r=1}^{N_2^n}\Bigg[\frac{1}{N_1^n}\left|\mathcal W^n(x^n_{1r})-\mathcal W^n(x^n_{N_1^n(r-1)})\right|^2
+\sum_{c=2}^{N_1^n}\frac1{N_1^n}|\mathcal W^n(x^n_{cr})-\mathcal W^n(x^n_{(c-1)r})|^2\Bigg]\\
&\geq\iota_A(\mathbf p^n)
+\frac{N_1^n}{2\diff\Delta t^n}\sum_{r=1}^{N_2^n}\Bigg|\frac{\mathcal W^n(x^n_{1r})-\mathcal W^n(x^n_{N_1^n(r-1)})
+\sum_{c=2}^{N_1^n}\mathcal W^n(x^n_{cr})-\mathcal W^n(x^n_{(c-1)r})}{N_1^n}\Bigg|^2\\
&=\iota_A(\mathbf p^n)
+\frac{\Delta x_2^n}{2\diff\Delta t^nN_1^n}\sum_{r=1}^{N_2^n}\frac{|\mathcal W^n(x^n_{N_1^nr})-\mathcal W^n(x^n_{N_1^n(r-1)})|^2}{\Delta x_2^n}\\
&=\iota_A(\mathbf p^n)
+\frac{\Delta x_2^n}{2\diff\Delta t^nN_1^n}\int_0^a|\partial_{x_2}\mathcal W^n(a,x_2)|^2\,\d x_2\,,
\end{align*}
where we used Jensen's inequality.
Due to the weak lower semi-continuity of the $H^1$-norm, we thus obtain $\liminf_{n\to\infty}\tilde{\mathcal E}^n\geq\iota_A(\mathbf p)+\frac{a}{2\diff T}\int_0^a|\partial_{x_2}\mathcal W(a,x_2)|^2\,\d x_2$.
In addition, this limit can even be achieved by choosing
\begin{equation*}
\mathbf p^n=\mathbf p\,,\quad
\mathcal W^n(x^n_{cr})=\mathcal W(x^n_{N_1^n(r-1)})+\frac c{N_1^n}[\mathcal W(x^n_{N_1^nr})-\mathcal W(x^n_{N_1^n(r-1)})]\,,\quad
c=1,\ldots,N_1^n,\,r=1,\ldots,N_2^n,
\end{equation*}
so that in above Jensen's inequality we actually have an equality.
Thus the non-data terms coincide in the $\Gamma-\liminf$ and the $\Gamma-\limsup$.
\end{proof}

The disparity between $\Gamma-\liminf$ and $\Gamma-\limsup$ implies that the functional does not have a $\Gamma$-limit as $n\to\infty$
and thus the density and motion reconstruction problems do not converge.
This does not necessarily imply that our thought experiment with faster and faster row scanning is completely unreasonable or unphysical;
instead a refinement of our stochastic coupling might be necessary to make sense of the limit.
Intuitively, in the limit each row is swept out in zero time so that the electron counts belonging to a row cannot be ascribed a particular horizontal position along the row.
In other words, $\elec^\infty$ can only capture information on the vertical, but not on the horizontal location of the detected electrons.
Resolving such additional information requires the use of a different, more involved stochastic coupling.
Instead of pursuing that route, we shall propose a slight model change below in which a measurement only reflects the average of the material density along each row.

\begin{remark}[Tomographic scanning models]\label{rem:tomography}
One may adapt our reconstruction model and in particular the data term by only estimating the average material density per row. The new energy then reads
\begin{equation*}
E^n[\motion,\mathbf p]
=\sum_{r=1}^{N_2^n}\diss^n\left(\sum_{c=1}^{N_1^n}(\mathbf\img^n)_{(r-1)N_1^n+c},\sum_{c=1}^{N_1^n}\Delta t^n\den[\mathbf p](x_{cr}^n+w_{(r-1)N_1^n+c})\right)+\frac1{2\diff}\sum_{i=1}^{N^n}\frac{|w_i-w_{i-1}|^2}{\Delta t^n}+\iota_A(\mathbf p)\,,
\end{equation*}
and for its continuous version
\begin{equation*}
\mathcal E^n[\mathcal W,\mathbf p]
=\begin{cases}
E^n[\motion,\mathbf p]&\text{if }\mathcal W:\Omega\to\R^2\text{ is the piecewise bilinear interpolation of the}\\&\quad\text{ values }\motion=(w_1,\ldots,w_{N^n})\text{ at positions }(x_{11}^n,x_{21}^n,\ldots,x_{N_1^n1}^n,x_{12}^n,\ldots,x_{N_1^nN_2^n}^n)\\
\infty&\text{else}
\end{cases}
\end{equation*}
one expects to obtain by a similar proof as above that $\Gamma-\lim_{n\to\infty}\mathcal E^n-C^n=\mathcal E^\infty$ for
\begin{multline*}
\mathcal E^\infty[\mathcal W,\mathbf p]
=\diss_{\mathrm{KL}}\left(\elec^\infty,t\mapsto\frac1a\int_0^a\den[\mathbf p]((x_1,vt)+\mathcal W(x_1,vt))\,\d x_1\right)\\
+\frac a{2\diff T}\int_0^a|\partial_{x_2}\mathcal W(a,x_2)|^2\,\d x_2+\iota_{\partial_{x_1}\mathcal W=0}(\mathcal W)
+\iota_A(\mathbf p)\,.
\end{multline*}
From the physical viewpoint, the limit problem would only allow to extract information of the average material density $\den$ along each horizontal line.
However, if such information were obtained for many different rotations of the sample,
then this would correspond to sampling the Radon transform of the material density at a number of angles.
The sought material density could thus be reconstructed by standard tomographic techniques as for instance used in computerized tomography.
\end{remark}

\section{An auxiliary convex model for atom identification}
\label{sec:ConvAtomIdent}
To solve our variational problem numerically, we require the number of atoms and their approximate positions as initialization.
To this end, we reduce our complex model to a smaller auxiliary optimization problem, which is desigend to be convex so that the global minimizer can be found.
This identification of the global optimum is important, since otherwise the non-convex main model would easily get stuck in local minima, compromising the accuracy of our material density reconstruction.

\subsection{Lifting atom positions in measure space}
To specify an appropriate model, note that it does not have to accurately describe all parts of the image acquisition, but rather should be as simple as possible.
Thus, we shall here assume that a single atom at position $\centerPos$ is described by a response function
\begin{equation*}
\bumpCoeff b(\cdot-\centerPos)\,,\quad
b\text{ fixed, \eg }
b(x)=\exp\left(\frac{-|x|^2}{2\bumpWidth^2}\right)\,,
\text{ and }c>0.
\end{equation*}
In contrast to before, each atom now just is represented by two parameters, position $\centerPos$ and height $\bumpCoeff$, while the atom width $\bumpWidth$ is fixed a priori by the user to a reasonable value.
Representing the distribution of atoms with positions $\centerPos_\atomIndex$ and heights $\bumpCoeff_\atomIndex$ via a sum of weighted Dirac measures,
\begin{equation*}
h=\sum_{\atomIndex=1}^\numAtoms\bumpCoeff_\atomIndex\delta_{\centerPos_\atomIndex}\,,
\end{equation*}
the corresponding image or material density can be expressed as the convolution
\begin{equation*}
\den=h*b\,.
\end{equation*}
We now introduce two further simplifications compared to our main model from Section\,\ref{sec:model}.
\begin{itemize}
\item
We will ignore the Brownian motion leading to misplaced pixels.
Instead, we interpret the slightly changed pixel intensities as noise rather than pixel displacements.
Since the smooth material density $\den$ can locally be approximated by an affine function,
the change in pixel intensity due to Brownian motion-induced pixel displacement becomes visible as Gaussian noise
(whose variance actually depends on the local slope of $\den$, but will in the following be assumed fixed for simplicity).
The additional mixed Poisson-Gaussian noise inherent in the signal detection (described in Section\,\ref{sec:noiseModel}) is typically of a much smaller size and thus will be ignored.
Therefore we will use the for Gaussian noise appropriate quadratic data term
\begin{equation*}
\sum_{i=1}^N\diss(\img_i,\den(x_i))\quad\text{ with }\diss(\img,z)=|z-\img|^2\,.
\end{equation*}
\item
Instead of optimizing over the atom positions $\centerPos_\atomIndex$, which would be a highly non-convex optimization,
we optimize directly for the measure $h$. In other words, we lift the vector of unknowns to a measure, thereby allowing a convex optimization.
Since this way the discrete nature of the atom positions is no longer strictly enforced, we have to add a regularization to our model that promotes spatial sparseness of the measure $h$.
$L^1$-type norms are widely used for this purpose, so we shall additionally penalize the total mass of $h$.
\end{itemize}
Summarizing, we shall solve the variational model
\begin{equation*}
\min_{h\in\measp(\Omega)}\mathcal F(h)\quad
\text{ for }\mathcal F(h)=\sum_{i=1}^N|(h*b)(x_i)-\img_i|^2+\eta\|h\|_{\measp(\Omega)}\,,
\end{equation*}
where $\measp(\Omega)$ denotes the set of nonnegative Radon measures and $\eta$ is some positive weight.
It is straightforward to prove the existence of a minimizer $h$ via the direct method of the calculus of variations.

\subsection{Extracting atoms from the lifting}
The resulting minimizer $h$  will only approximately represent a linear combination of Dirac measures
so that some postprocessing is required to extract the atom positions and heights.
In detail, we identify all connected components $C_\atomIndex\subset\Omega$, $\atomIndex=1,\ldots,\numAtoms$, of the support of $h$ (which is readily done with a computational complexity proportional to the number of discretization points)
and define the atom positions as
\begin{equation*}
\centerPos_\atomIndex=\int_{C_\atomIndex}x\,\d h(x)\bigg/\int_{C_\atomIndex}\,\d h(x)\,,\quad\atomIndex=1,\ldots,\numAtoms\,.
\end{equation*}
Furthermore, we set the atom heights to
\begin{equation*}
\bumpCoeff_\atomIndex=\int_{C_\atomIndex}\,\d h(x)+\frac\eta2\bigg/\sum_{i=1}^Nb(\centerPos_\atomIndex-x_i)^2\,,\quad\atomIndex=1,\ldots,\numAtoms\,.
\end{equation*}
The latter equation is motivated by the simple fact that for a single atom ground truth, $\img_i=\bumpCoeff b(\centerPos-x_i)$,
if atom position $\centerPos$ is known, then $h=(\bumpCoeff-\frac\eta2/\sum_{i=1}^Nb(\centerPos-x_i)^2)\delta_{\centerPos}$ minimizes the energy $\mathcal F$ among all multiples of $\delta_{\centerPos}$.
Finally, in order to eliminate atoms that were just introduced by the minimization in order to reproduce background variations of the image, we simply remove all atoms with height below a manually specified threshold.

\subsection{Numerical optimization by a semi-smooth Newton method}
According to the first order optimality conditions for minimizing the convex functional $\mathcal F$, the subdifferential $\partial\mathcal F(h)$ must contain $0$.
Denoting by $\iota_{\measp(\Omega)}$ the indicator function of $\measp(\Omega)$, we thus obtain
\begin{equation*}
0\in\sum_{i=1}^N2((h*b)(x_i)-\img_i)b(x_i-\cdot)+\eta+\partial\iota_{\measp(\Omega)}(h)
\end{equation*}
or equivalently
\begin{equation*}
-\beta\in\partial\iota_{\measp(\Omega)}(h)\quad\text{ for }\beta(x)=\sum_{i=1}^N2h((h*b)(x_i)-\img_i)b(x_i-x)+\eta
\end{equation*}
and thus (denoting by $\langle\cdot,\cdot\rangle$ the dual pairing between $\meas(\Omega)$ and its dual, and interpreting $\beta$ as element of $\meas(\Omega)^\prime$ via $\langle\beta,\tilde h\rangle=\int\beta\mathrm{d}\tilde h$)
\begin{equation*}
\langle\beta,\tilde h-h\rangle\geq0\qquad\text{for all }\tilde h\in\measp(\Omega)\,.
\end{equation*}
Since $\beta$ is smooth, this is equivalent to $\beta\geq0$ and $\beta=0$ on $\spt(h)$. Thus, for any $\gamma>0$ we have
\begin{equation*}
0=h-\max(0,h-\gamma\beta)
=h-\max\left(0,h-\gamma\left[2\sum_{i=1}^N((h*b)(x_i)-\img_i)b(x_i-\cdot)+\eta\right]\right)\,,
\end{equation*}
to be solved for the measure $h$, where $\max(0,\cdot)$ denotes the positive part of a measure according to the Hahn decomposition theorem.

The aim is to solve the above equation via Newton's method, however,
the equation is well-known to lack the necessary semi-smoothness on the space $\meas(\Omega)$ of Radon measures (see \eg \cite[Lem.\,2.7]{HiPiUlUl09} for the corresponding argument in the $L^q$-setting).
Note, though, that we may very coarsly discretize the measure $h$ as we do not require high accuracy for the solution of the auxiliary problem.
In fact, we only have to resolve approximately half an atom width in order to still be able to identify all atoms.
Thus, we choose a sqaure grid on $\Omega$ of grid width $\Delta x$ and define $\chi_{kl}$ to be the characteristic function on the square in the $k$\textsuperscript{th} row and $l$\textsuperscript{th} column.
Let
\begin{equation*}
h=\sum_{k,l}h_{kl}\xi_{kl}\quad\text{ for }\xi_{kl}=\frac{\chi_{kl}}{\Delta x^2}
\end{equation*}
and denote the vector of coefficients $h_{kl}$ by $\mathbf h$. Again deriving the optimality conditions for $\mathbf h$ we arrive at
\begin{equation*}
0=h_{kl}-\max\left(0,h_{kl}-\gamma\left[2\sum_{i=1}^N\left(\sum_{m,n}h_{mn}\int_{\R^2}\xi_{mn}(x)b(x_i-x)\,\d x-\img_i\right)\int_{\R^2}\xi_{kl}(x)b(x_i-x)\,\d x+\eta\right]\right)=:F(\mathbf h)_{kl}
\end{equation*}
for all $k,l$.
This is now solved for $\mathbf h$ using a semi-smooth Newton method,
where the Newton equation
\begin{equation*}
DF(\mathbf h^{\mathrm{old}})
(\mathbf h^{\mathrm{old}}-\mathbf h^{\mathrm{new}})
=F(\mathbf h^{\mathrm{old}})
\end{equation*}
in each iteration is solved by a few steps of the GMRES method.
Here, the generalized differential of $F$ is given by
\begin{equation*}
D(F(\mathbf h)_{kl})_{mn}
=\begin{cases}
\delta_{kl,mn}&\text{if }F(\mathbf h)_{kl}=h_{kl}\\
2\gamma\sum_{i=1}^N\int_{\R^2}\xi_{mn}(x)b(x_i-x)\,\d x\int_{\R^2}\xi_{kl}(x)b(x_i-x)\,\d x&\text{else.}
\end{cases}
\end{equation*}

Note that in case of a scanning path with scan positions at points $x_{ij}$ of the same rectangular grid, the necessary computations can efficiently be performed using the fast Fourier transform (FFT).
Indeed, let
\begin{equation*}
b_{-i,-j}=\int_{\R^2}\xi_{ij}(x)b(x_{00}-x)\,\d x\quad\text{ for all }i,j\in\Z\,,
\quad\mathbf b=(b_{ij})_{i,j\in\Z}\,,
\end{equation*}
where $x_{00}$ uses the canonical of our grid indexing, \ie $x_{00}:=x_{11}-(\Delta x,\Delta x)$.
Then,  we have
\begin{equation*}
\sum_{m,n}h_{mn}\int_{\R^2}\xi_{mn}(x)b(x_{ij}-x)\,\d x
=\sum_{m,n}h_{mn}b_{i-m,j-n}
=(\mathbf h*\mathbf b)_{ij}\,,
\end{equation*}
and this discrete convolution can be computed efficiently using FFT.
Furthermore, $F$ becomes
\begin{equation*}
F(\mathbf h)_{kl}
=h_{kl}-\max\left(0,h_{kl}-\gamma\left[2(\left(\mathbf h*\mathbf b-\mathbf\img\right)*\bar{\mathbf b})_{kl}+\eta\right]\right)
\quad\text{ for }\mathbf\img=(\img_{ij})_{i,j\in\N}\,,\quad\bar{\mathbf b}_{ij}=(b_{-i,-j})_{i,j\in\N}\,,
\end{equation*}
with
\begin{equation*}
D(F(\mathbf h)_{kl})_{mn}
=\begin{cases}
\delta_{kl,mn}&\text{if }F(\mathbf h)_{kl}=h_{kl}\\
2\gamma(\mathbf b*\bar{\mathbf b})_{k-m,l-n}&\text{else.}
\end{cases}
\end{equation*}

In practice, the method robustly identifies all atoms within an image with a computation time far below that of the full model (tens of seconds).
An example is provided in Figure~\ref{fig:atomInitialization}, where the atom locations are extracted from a STEM image of GaN.

\begin{figure}
\setlength\unitlength{.25\linewidth}%
\fboxsep0pt%
\centering\hfill
\includegraphics[width=\unitlength,type=png,ext=.png,read=.png]{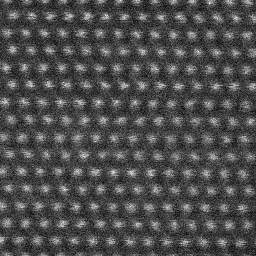}\hfill
\framebox{\includegraphics[width=\unitlength]{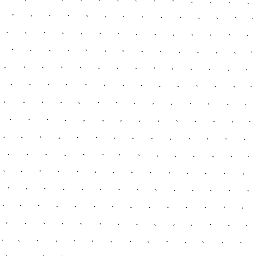}}\hfill
\includegraphics[width=\unitlength]{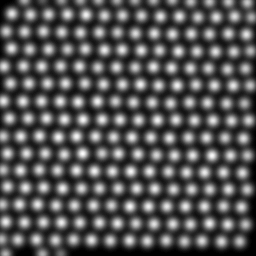}\hfill\hfill
\caption{From left to right: STEM image of GaN (courtesy of Paul~M.\ Voyles), computed atom distribution $h$, resulting initial guess for the material density $\den$.}
\label{fig:atomInitialization}
\end{figure}

\section{Numerical implementation}\label{sec:numerics}

While in the previous sections, for notational simplicity we have mainly considered the situation of a single input image $\mathbf\img$, we shall from now on consider the practically more relevant case of using the model from Section~\ref{sec:scanMode} with multiple input measurements $\mathbf\img^1,\ldots,\mathbf\img^K$.

\subsection{Necessary preprocessing of input images}
In case of multiple input measurements, the images $\mathbf\img^1,\ldots,\mathbf\img^K$ typically do not all show the same region of the material sample.
This needs to be corrected for via a data preprocessing step.
Applying the approach from Section~\ref{sec:ConvAtomIdent} on $\inputImage^1$, we get an initial guess $\imageParamVec^\text{ini}=(\centerPos_1^\text{ini},\ldots,\centerPos_\numAtoms^\text{ini},\bumpCoeff_1^\text{ini}\ldots,\bumpCoeff_\numAtoms^\text{ini},\bumpWidth^\text{ini},o^\text{ini})$ for the parameters of $\reconstImage$. Then, the input images $\inputImage^2,\ldots,\inputImage^K$ are aligned with $\inputImage^1$ using standard image alignment techniques, thus bringing all images into the coordinate system of $\inputImage^1$. In our experiments we do so by determining the optimal translation (as integer pixel shift) using the series registration strategy from \cite{BeBiBl13} but with a deformation model that just allows translations. Since the shifted images are not defined at those pixels that have no correspondence in their original coordinate system, the aligned images are cropped to their common support. For the sake of simplicity, the aligned, cropped images are again denoted by $\inputImage^1,\ldots,\inputImage^K$ and their size by $\numPixelsX\times\numPixelsY$. The initial center positions $\centerPos_1^\text{ini},\ldots,\centerPos_\numAtoms^\text{ini}$ are adjusted accordingly.
Furthermore, we remove all atoms from $\imageParamVec^\text{ini}$ which lie far enough outside the new, cropped image domain $\Omega$ such that they at most contribute a value of $10^{-8}$ to $\den[\imageParamVec^\text{ini}]$ inside $\Omega$ (those are exactly the atoms with center further away from $\Omega$ than
$\sqrt { -2(\bumpWidth^\text{ini})^2\log \left(\frac{ 10^{-8}}{ ||\inputImage^1||_\infty} \right)  }$). In different words, we keep atoms at positions that are very close to but slightly outside of the domain.
This way, we avoid that atoms inside $\Omega$ try to compensate for atoms outside $\Omega$ that are neglected (since they are not visible in $\Omega$) but still account for a few deflected electron counts measured inside $\Omega$.

\subsection{Additional, heuristic model modifications}
Since the atoms in $\imageParamVec^\text{ini}$ but outside $\Omega$ are not visible in the cropped images, we will penalize them to stay close to their initially estimated position using the penalty
\[
P[\imageParamVec]=\frac{\nu_\text{pen}}2\sum_{\centerPos_\atomIndex^\text{ini}\notin\Omega} \left[(\bumpCoeff_\atomIndex-\bumpCoeff^\text{ini}_\atomIndex)^2+|\centerPos_\atomIndex-\centerPos^\text{ini}_\atomIndex|^2\right]\,,
\]
where $\nu_\text{pen}>0$ is a constant weighting factor.

Using the STEM rastering pattern~\eref{eq:STEMIndexing} and defining the initial accumulated Brownian motion as zero, the Brownian motion term in \eref{eqn:objectiveMultipleInputs} for $\shiftVec\in(\mathbb{R}^2)^{\numPixelsX\times\numPixelsY}$ becomes
\begin{equation*}
R_1[\shiftVec]=\frac1{2\diff}\left[\frac{|\shift_{11}|^2}{\Delta t}+\sum_{j=2}^\numPixelsY\left(\frac{|\shift_{1j}-\shift_{\numPixelsX(j-1)}|^2}{\Delta T}+\sum_{i=2}^\numPixelsX\frac{|\shift_{ij}-\shift_{(i-1)j}|^2}{\Delta t}\right)\right]\,.
\end{equation*}
During the numerical experiments, we found that penalizing the squared norm of the random motion $\shift_{ij}$ with a small weight noticeably improves the convergence speed without influencing the results much. The reason lies in the slightly stronger local convexity of the resulting energy functional. Indeed, without this penalization a constant shift of all $\shift_{ij}$ by some $z\in\R^2$ could be compensated for via shifting all atoms by $z$ as well, indicating a lack of local positive definiteness of the energy Hessian. Thus, we add the regularizer
\begin{equation*}
R_2[\shiftVec]=\frac1{2}\sum_{i=1}^\numPixelsX\sum_{j=1}^\numPixelsY\left(\nu_\text{hor}|(\shift_{ij})_1|^2+\nu_\text{vert}|(\shift_{ij})_2|^2\right),
\end{equation*}
where $(\shift_{ij})_1$ and $(\shift_{ij})_2$ denote the horizontal and vertical component of $\shift_{ij}\in\mathbb{R}^2$ and $\nu_\text{hor},\nu_\text{vert}>0$ are constant weighting factors. Depending on the input data, it may even be beneficial to use a noticeably larger value for one of the weights. For instance, the data shown in Figure~\ref{fig:STEMAndSimulImage} has numerous pixel rows that contain almost no signal. Hence, the energy landscape is very flat with respect to the vertical shift of these rows so that during optimization this vertical shift often ends up fitting the noise (note that this is not a problem of our underlying model, but rather a general problem of MAP estimates in regions of relatively constant probability density). This can be prevented by an increased value of $\nu_\text{vert}$.

In summary, for numerical experiments we employ the following heuristically improved version of \eref{eqn:objectiveMultipleInputs},
\[
E^K[\shiftVec^1,\ldots,\shiftVec^K,\imageParamVec]
=\sum_{k=1}^K\bigg[
\sum_{i=1}^\numPixelsX\sum_{j=1}^\numPixelsY d(\reconstImage(x_{ij}+\shift_{ij}^k),(\inputImage^k)_{ij})+R_1[\shiftVec^k]+R_2[\shiftVec^k]
\bigg]+P[\imageParamVec]\,.
\]

\subsection{Numerical optimization}
Note that $P[\imageParamVec]$, $R_1[\shiftVec]$ and $R_2[\shiftVec]$ are quadratic in their arguments, so the first and second derivatives can be computed easily. Using $z=x_{ij}+\shift_{ij}^k$ to shorten the notation, and calling the data term $D$, the derivatives of $D$ are as follows:
\begin{align*}
\nabla_{\imageParamVec}D[\shiftVec^1,\ldots,\shiftVec^K,\imageParamVec]
={}&\sum_{k=1}^K
\sum_{i=1}^\numPixelsX\sum_{j=1}^\numPixelsY \partial_1d(\reconstImage(z),(\inputImage^k)_{ij})\nabla_{\imageParamVec}\reconstImage(z)\\[1ex]
\nabla_{\shift_{ij}^k}D[\shiftVec^1,\ldots,\shiftVec^K,\imageParamVec]
={}&\partial_1d(\reconstImage(z),(\inputImage^k)_{ij})\nabla_x\reconstImage(z)\\[1ex]
\nabla^2_{\shift_{ij}^k}D[\shiftVec^1,\ldots,\shiftVec^K,\imageParamVec]
={}&\partial^2_1d(\reconstImage(z),(\inputImage^k)_{ij})\nabla_x\reconstImage(z)\otimes\nabla_x\reconstImage(z)\\
&+\partial_1d(\reconstImage(z),(\inputImage^k)_{ij})\nabla^2_x\reconstImage(z)\\[1ex]
\nabla_{\shift_{ij}^k}\nabla_{\imageParamVec}D[\shiftVec^1,\ldots,\shiftVec^K,\imageParamVec]
={}&\partial^2_1d(\reconstImage(z),(\inputImage^k)_{ij})\nabla_x\reconstImage(z)\otimes\nabla_{\imageParamVec}\reconstImage(z)\\
&+\partial_1d(\reconstImage(z),(\inputImage^k)_{ij})\nabla_x\nabla_{\imageParamVec}\reconstImage(z)\\[1ex]
\nabla^2_{\imageParamVec}D[\shiftVec^1,\ldots,\shiftVec^K,\imageParamVec]
={}&\sum_{k=1}^K
\sum_{i=1}^\numPixelsX\sum_{j=1}^\numPixelsY \bigg[\partial^2_1d(\reconstImage(z),(\inputImage^k)_{ij})\nabla_{\imageParamVec}\reconstImage(z)\otimes \nabla_{\imageParamVec}\reconstImage(z)\\
&\qquad\qquad\qquad+\partial_1d(\reconstImage(z),(\inputImage^k)_{ij})\nabla^2_{\imageParamVec}\reconstImage(z)\bigg]
\end{align*}
The functional $E^K$ is minimized using the trust region Newton method from~\cite{Co00}.
In particular, for the trust region subproblem, we chose to implement the algorithm proposed in \cite[Algorithm 7.3.4]{Co00},
making use of a Cholesky factorization of the energy Hessian and an eigendirection-based approach to bypass saddle points, where the Cholesky factorization is performed using the CHOLMOD package from Davis\,et\,al.\,\cite{ChDaHa09}.
To start this method with an initial guess $\imageParamVec$ slightly improved over $\imageParamVec^\text{ini}$, we proceed as follows.

\begin{itemize}
\item Find $\imageParamVec$ by minimizing
$
\sum_{i=1}^\numPixelsX\sum_{j=1}^\numPixelsY d(\reconstImage(x_{ij}),(\inputImage^1)_{ij})+P[\imageParamVec]
$
with respect to $\imageParamVec$ using $\imageParamVec^\text{ini}$ as initial value, but keeping the shared bump width $\bumpWidth$ fixed.
To this end we use the computationally rather cheap BFGS quasi-Newton method.
\item Further refine $\imageParamVec$ by minimizing
$
\sum_{k=1}^K\bigg[
\sum_{i=1}^\numPixelsX\sum_{j=1}^\numPixelsY d(\reconstImage(x_{ij}),(\inputImage^k)_{ij})
\bigg]+P[\imageParamVec]
$
with respect to $\imageParamVec$ using the trust region Newton method.
\item Improve the initial guess $\shiftVec^1=\ldots=\shiftVec^K=0$ and refine $\imageParamVec$ by minimizing $E^K[\shiftVec^1,\ldots,\shiftVec^K,\imageParamVec]$ using the trust region Newton method while enforcing that $\shiftVec^k$ are constant in each pixel row.
\end{itemize}
Finally, we can minimize the whole functional $E^K$ using the trust region Newton method starting from the initial guess obtained above.

\section{Experiments on synthetic and real data}\label{sec:results}
All numerical experiments shown in this sections were conducted with the following reconstruction parameters:
$\diff=0.1$, $\Delta t=1/\max(\numPixelsX-1,\numPixelsY-1)$, $\Delta T=1000\Delta t$, $\nu_\text{pen}=0.05$, $\nu_\text{hor}=0.1$ and $\nu_\text{vert}=10$. Note that the relatively large value of $\Delta T$ is chosen to compensate for the so-called flyback error: when the electron probe is instructed to jump back from the end of one scan line to the beginning of the next line, the probe will not be positioned exactly where it is supposed to be. The position error from this effect is larger than the error caused by the Brownian motion during the time between ending one line and beginning the scan in the next line. However, it can be modeled as Brownian motion during a larger, ficticious time interval $\Delta T$ for which reason we did not explicitly include the flyback error in our forward model.

\subsection{Synthetic data generation and experiments}
To test the effectiveness of the proposed method, we first apply it to synthetic data for which the ground truth is known.
To this end, we consider a ground truth image $\den$ with Gaussian bump functions of height $45$ and standard deviation of $3$\,pixel units, arranged in a hexagonal lattice of $19.37$\,pixel units lattice spacing.
We subsequently discretize the image into $256\times256$ pixels, however, the position of the $i$\textsuperscript{th} pixel (counting row-wise) is displaced by a vector $w_i=\sum_{j=1}^i\Delta_j+\sum_{j=1}^{i/256}\hat\Delta_j$.
Here, $\Delta_1,\Delta_2,\ldots$ are random numbers drawn from a normal distribution with mean 0 and standard deviation of 0.05\,pixel units,
while $\hat\Delta_1,\hat\Delta_2,\ldots$ represent the Brownian motion happening between two pixel rows and thus are drawn from a normal distribution with mean 0 and a larger standard deviation of 1\,pixel unit.
Finally, a constant background of $o=40$ is added, and the image is corrupted by Poisson noise.
In this way, a series of 128 randomly corrupted images $\mathbf\img^1,\ldots,\mathbf\img^{128}$ is generated.
Figure~\ref{fig:STEMAndSimulImage} displays the first synthetic image $\mathbf\img^1$ next to a real image acquired by STEM, showing good qualitative agreement.

\begin{figure}
\centering
\begin{tabular}{cc}
\includegraphics[width=0.45\linewidth,type=png,ext=.png,read=.png]{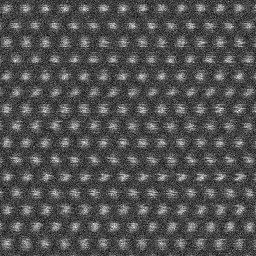}&
\includegraphics[width=0.45\linewidth,type=png,ext=.png,read=.png]{figures/3_8_12-GaN-Area2_20.5_HAADF_1000_counts}
\end{tabular}
\caption{Synthetic image created via our proposed image formation model (left) and an experimental STEM image of GaN (right, courtesy of Paul~M.\ Voyles).}
\label{fig:STEMAndSimulImage}
\end{figure}

Figure~\ref{fig:syntheticData} shows the reconstruction results of our algorithm for four synthetic input images (left column);
in addition to $\reconstImage$ (bottom row), we display a color coding of the computed random pixel displacements $\motion^k$ (middle column) as well as those displacements applied to $\reconstImage$ (right column), which should represent the input images without the Poisson noise in each pixel.
While the detected random displacement may vary considerably between consecutive pixel rows, it seems to be relatively constant along each row.
However, the plot of the displacement along a few selected pixel rows in Figure~\ref{fig:syntheticShift} reveals a Brownian-motion-like variation along a row as well.

\begin{figure}
\centering
\setlength{\unitlength}{.72\linewidth}
\begin{tabular}{rccc}
&
$\inputImage^k$&
$\shiftVec^k$&
${\small\reconstImage(x_{ij}+\shift_{ij}^k)}$
\\[1ex]
\raisebox{.145\unitlength}{$k=1$}&
\includegraphics[width=0.3\unitlength]{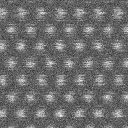}&
\includegraphics[width=0.3\unitlength]{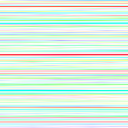}&
\includegraphics[width=0.3\unitlength]{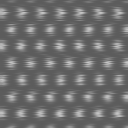}\\
\raisebox{.145\unitlength}{$k=2$}&
\includegraphics[width=0.3\unitlength]{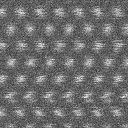}&
\includegraphics[width=0.3\unitlength]{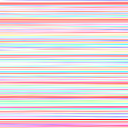}&
\includegraphics[width=0.3\unitlength]{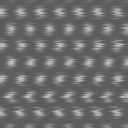}\\
\raisebox{.145\unitlength}{$k=3$}&
\includegraphics[width=0.3\unitlength]{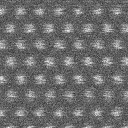}&
\includegraphics[width=0.3\unitlength]{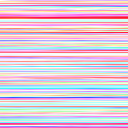}&
\includegraphics[width=0.3\unitlength]{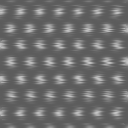}\\
\raisebox{.145\unitlength}{$k=4$}&
\includegraphics[width=0.3\unitlength]{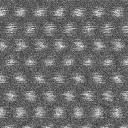}&
\includegraphics[width=0.3\unitlength]{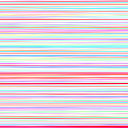}&
\includegraphics[width=0.3\unitlength]{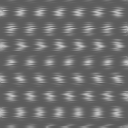}\\
\hline
$\vphantom{\displaystyle\sum}$&
$\reconstImage-o$&
&$\reconstImage$\\
&
\includegraphics[width=0.3\unitlength]{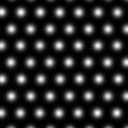} &
\begin{picture}(.3,.3)\put(.1,.1){\includegraphics[width=0.1\unitlength]{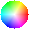}}\put(.15,.15){\thicklines\color{white}\scalebox{5}{\circle{.02}}}\end{picture}&
\reflectbox{\rotatebox[origin=c]{180}{\includegraphics[width=0.3\unitlength]{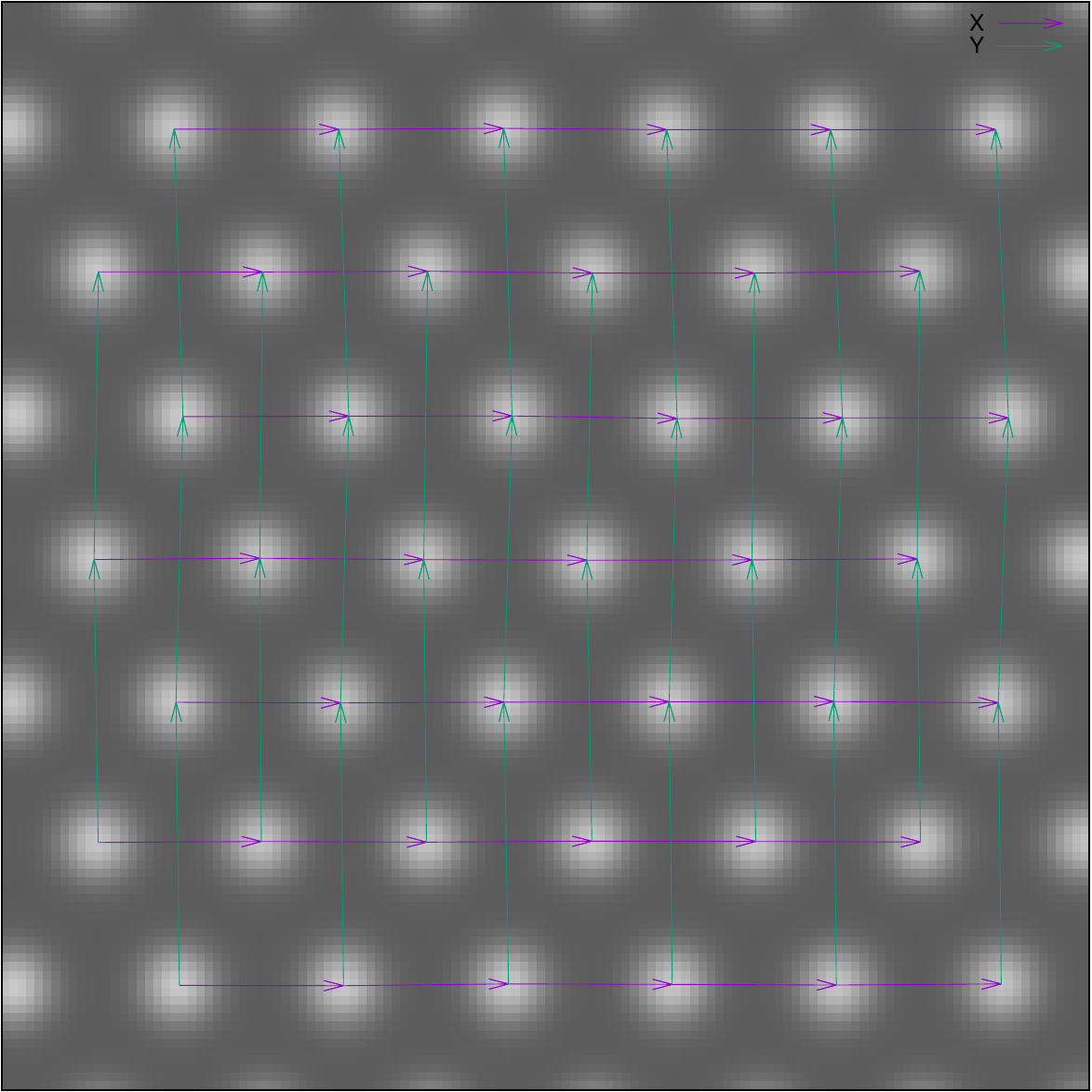}}}
\end{tabular}
\caption{Given four synthetic input images $\inputImage^1,\ldots,\inputImage^4$ (left column), our algorithm reconstructs the correct atom distribution $\reconstImage$ (bottom row)
as well as the random pixel displacement due to Brownian motion (middle column, color-coded according to the color wheel).
The displacements applied to the reconstruction are shown as well (right column).
}
\label{fig:syntheticData}
\end{figure}

\begin{figure}
\includegraphics{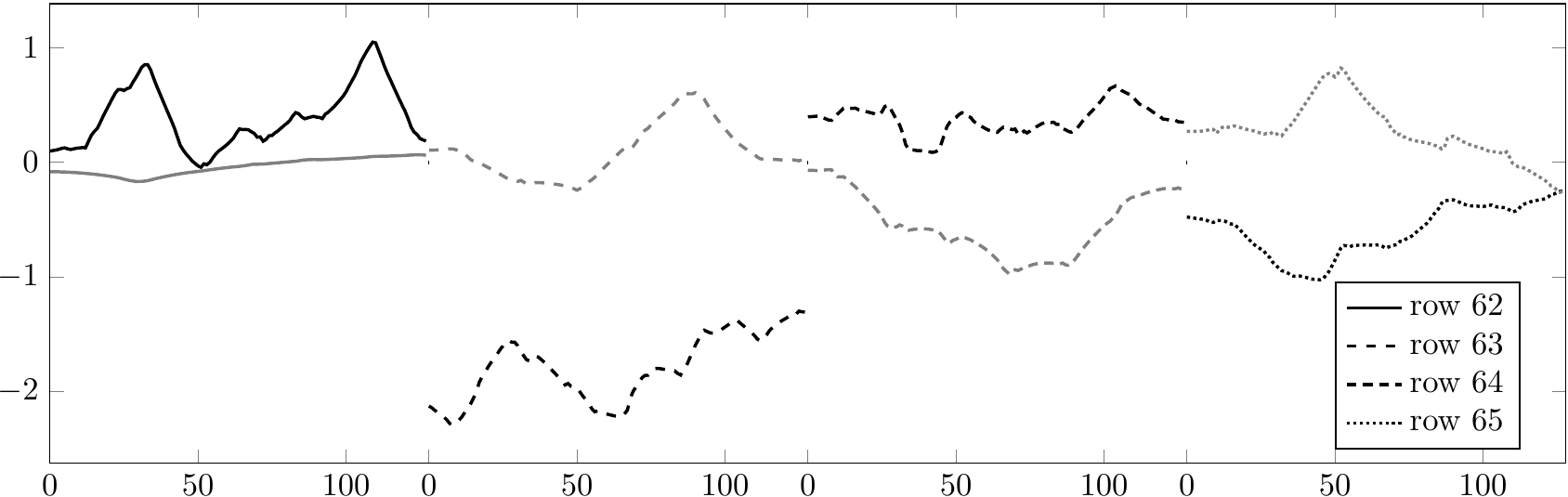}
\caption{Plots of the horizontal (black) and vertical (gray) component of $\shiftVec^1$ from Figure~\ref{fig:syntheticData} for selected pixel rows. The displacement is shown in pixels.}
\label{fig:syntheticShift}
\end{figure}

Using the synthetic data, we now perform the following series of experiments:
For each $k=0,\ldots,6$ we split the set of 128 synthetic input images into $2^{7-k}$ groups of each $2^k$ images.
To each such group we apply our reconstruction algorithm.
The quality of the result is measured via the so-called horizontal and vertical precision, which is the standard deviation of the horizontal and vertical atom distances and which should ideally be zero. The bottom right image of Figure~\ref{fig:syntheticData} illustrates which distances are used to compute the precision.
To avoid confusion let us emphasize that in this terminology low precision values correspond to high accuracy (so that the term ``precision'' may be a bit misleading).
We choose the precision as quality measure, since it is already well-established \cite{YaBeDa14} and also applicable to measurements without ground truth.
Figure~\ref{fig:syntheticPrecision} shows the horizontal and vertical precision, averaged over the performed experiments, as a function of the number $K$ of input images.
The precision decreases roughly like $K^{-1/2}$, which is the expected rate if the atom positions in each input image are displaced by Gaussian noise.
Note that the vertical precision almost is up to three times worse than the horizontal one.
This is expected, since the hexagonal grid of our synthetic images is aligned such that each pixel row traverses several atoms.
Thus, the random displacement changes much less between horizontally neighboring atoms than between vertically neighboring atoms (which lie in different pixel rows and thus have a much stronger random displacement between them).
Moreover, the absolute vertical distance of atom pairs is more than twice as big as the horizontal distance in this grid (cf.\ bottom right of Figure~\ref{fig:syntheticData}), which means that the absolute precisions should not be compared directly.
Nevertheless, even the vertical precision reaches values of 0.1\,pixels accuracy at 64 input images.

\begin{figure}
\centering
\includegraphics{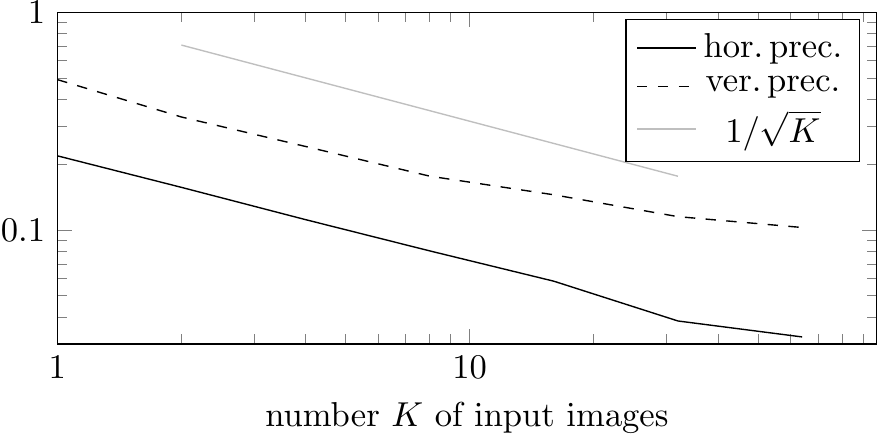}
\caption{Average horizontal and vertical precision (in pixels) of the reconstructed atom positions as a function of the number of synthetic input images.
}
\label{fig:syntheticPrecision}
\end{figure}

\subsection{Experiments on real data}
Figures\,\ref{fig:realData} to \ref{fig:realPrecision} show the same as Figures\,\ref{fig:syntheticData} to \ref{fig:syntheticPrecision}, only this time for real data acquired by STEM from a GaN material sample (experimental data courtesy of Paul M.\ Voyles).
For this data, the precision reaches a slightly higher value of 0.15\,pixels for 64 input images, which at the used STEM resolution corresponds to about $2.25$\,pm (pixel size in this case is about $15$\,pm).
Interestingly, the rate of precision decrease is lower than the expected $K^{-1/2}$ from the synthetic experiments.
Preliminary tests with different experimental data sets suggest that the origin of this inferior rate lies in the input data rather than the reconstruction method,
since it turns out that the better rate can be restored if the input images $\mathbf\img^1,\ldots,\mathbf\img^K$ are taken from different material regions (which nevertheless show the exactly same hexagonal atomic grid).
This fact actually speaks in favor of the reconstruction method, since it is apparently even able to identify hidden deviations of the measurements from a perfectly regular atomic grid.

\begin{figure}
\centering
\setlength{\unitlength}{.72\linewidth}
\begin{tabular}{rccc}
&
$\inputImage^k$&
$\shiftVec^k$&
${\small\reconstImage(x_{ij}+\shift_{ij}^k)}$
\\[1ex]
\raisebox{.145\unitlength}{$k=1$}&
\includegraphics[width=0.3\unitlength]{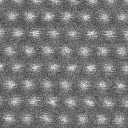}&
\includegraphics[width=0.3\unitlength]{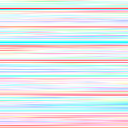}&
\includegraphics[width=0.3\unitlength]{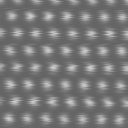}\\
\raisebox{.145\unitlength}{$k=2$}&
\includegraphics[width=0.3\unitlength]{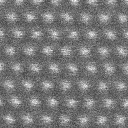}&
\includegraphics[width=0.3\unitlength]{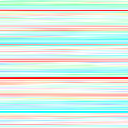}&
\includegraphics[width=0.3\unitlength]{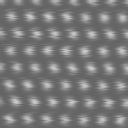}\\
\raisebox{.145\unitlength}{$k=3$}&
\includegraphics[width=0.3\unitlength]{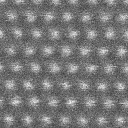}&
\includegraphics[width=0.3\unitlength]{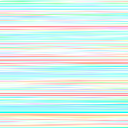}&
\includegraphics[width=0.3\unitlength]{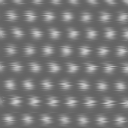}\\
\raisebox{.145\unitlength}{$k=4$}&
\includegraphics[width=0.3\unitlength]{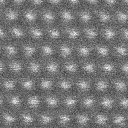}&
\includegraphics[width=0.3\unitlength]{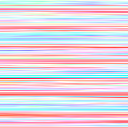}&
\includegraphics[width=0.3\unitlength]{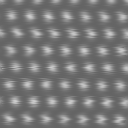}\\
\hline
$\vphantom{\displaystyle\sum}$&
$\reconstImage-o$&
&$\reconstImage$\\
&
\includegraphics[width=0.3\unitlength]{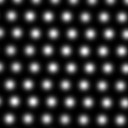} &
\begin{picture}(.3,.3)\put(.1,.1){\includegraphics[width=0.1\unitlength]{colWheel}}\put(.15,.15){\thicklines\color{white}\scalebox{5}{\circle{.02}}}\end{picture}&
\reflectbox{\rotatebox[origin=c]{180}{\includegraphics[width=0.3\unitlength]{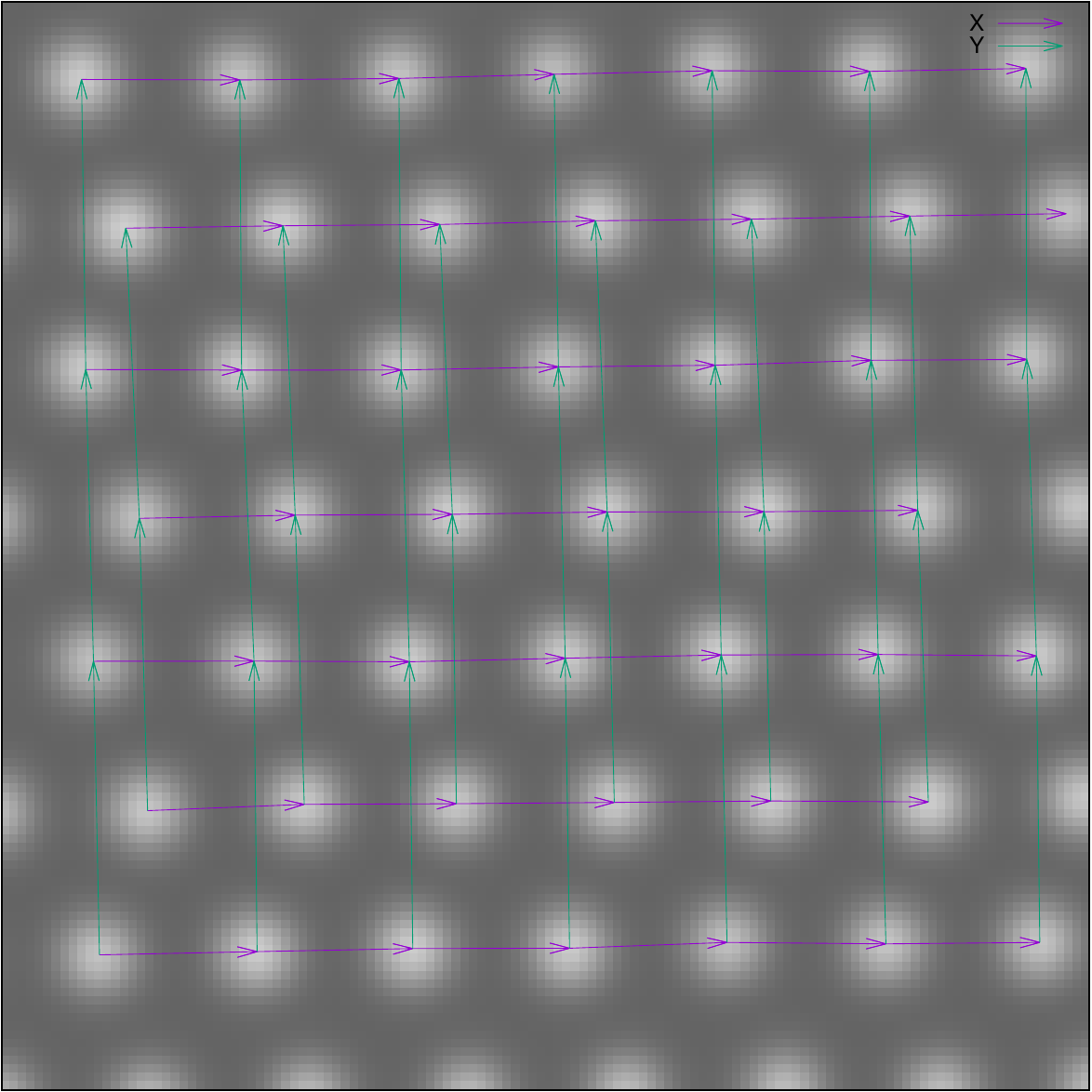}}}
\end{tabular}
\caption{Given four STEM images $\inputImage^1,\ldots,\inputImage^4$ of GaN (left column), our algorithm reconstructs the correct atom distribution $\reconstImage$ (bottom row)
as well as the random pixel displacement due to Brownian motion (middle column, color-coded according to the color wheel).
The displacements applied to the reconstruction are shown as well (right column).
}
\label{fig:realData}
\end{figure}

\begin{figure}
\centering
\includegraphics{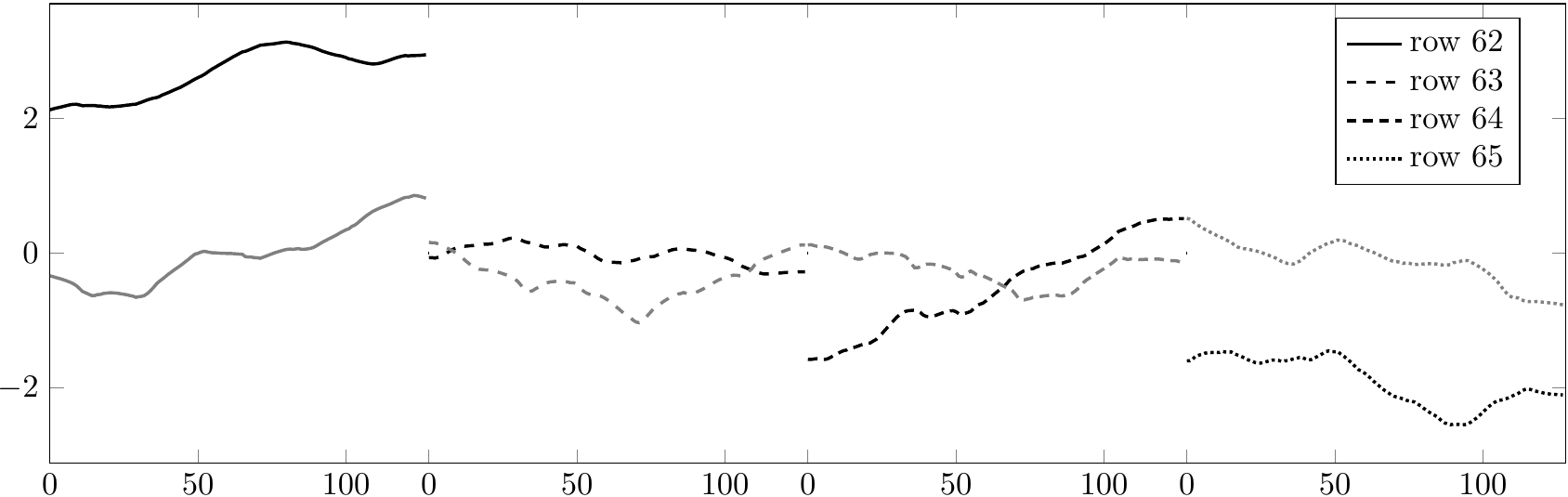}
\caption{Plots of the horizontal (black) and vertical (gray) component of $\shiftVec^1$ from Figure~\ref{fig:realData} for selected pixel rows. The displacement is shown in pixels.}
\label{fig:realShift}
\end{figure}

\begin{figure}
\centering
\includegraphics{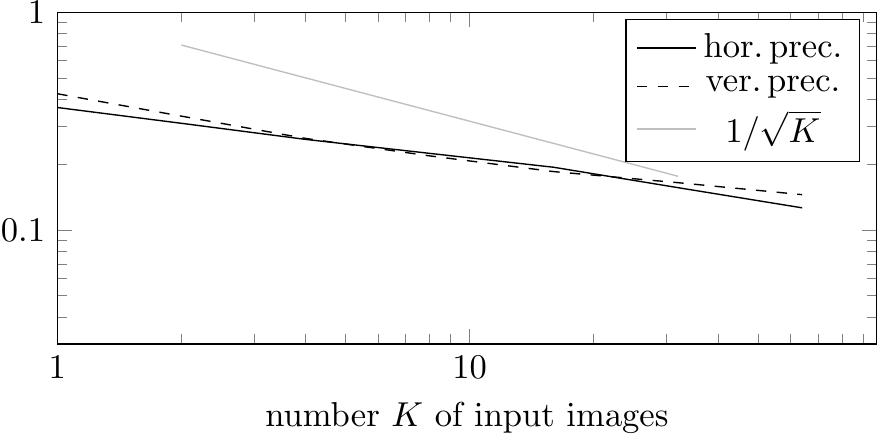}
\caption{Average horizontal and vertical precision (in pixels) of the reconstructed atom positions as a function of the number of STEM input images.
}
\label{fig:realPrecision}
\end{figure}

\section{Conclusions}
Using a Bayesian approach, we have derived a variational method to extract atom positions and random pixel displacement from STEM images.
With the help of tools from stochastic homogenization it turns out that one can also make sense of this reconstruction method if the electron beam traverses the scanned material in a time-continuous manner, constantly detecting deflected electrons.
Depending on the scanning path this may yield reconstructions based on space-filling measurements (that is, measurements at every single location in the 2D scanning domain $\Omega$)
or on tomography-type measurements.
To our knowledge, such approaches and scanning paths are not yet investigated in the STEM community, and it would be interesting to test how well they could be realized and what applications might benefit from them.

For a numerical implementation, we first derived a reduced, convex model that was used to provide an initialization for the full reconstruction method.
In our eyes, this (not new) idea is useful in inverse problems beyond STEM reconstruction:
Since accurate forward models often are complex and nonlinear, variational reconstructions may easily get stuck in suboptimal local minima;
efficient convex models can help out by providing a good, globally optimized initial guess.

Despite using a relatively accurate model of STEM imaging, choosing model parameters for a particular type of experimental images still requires a little tuning
(which is not problematic from the application viewpoint, since once the parameters are chosen one can use them for a large series of physical experiments).
In particular, some additional heuristic regularization seems to be beneficial.
This may be a manifestation of one of the well-known deficiencies of the MAP estimate for inverse problems:
If the energy landscape is very flat, the MAP estimate may tend to overfit noise, which can be counteracted by additional regularization.

\section*{Acknowledgements}
B.\ Berkels was funded in part by the Excellence Initiative of the German Federal and State Governments.
B.\ Wirth's research was supported by the Alfried Krupp Prize for Young University Teachers awarded by the Alfried Krupp von Bohlen und Halbach-Stiftung.
The work was also supported by the Deutsche Forschungsgemeinschaft (DFG), Cells-in-Motion Cluster of Excellence (EXC1003 -- CiM), University of M\"unster, Germany.

\bibliographystyle{plain}
\bibliography{paper}

\end{document}